\documentclass{article} 
\usepackage{iclr2024_conference,times}

\usepackage{amsmath,amsfonts,bm}

\def\eqref#1{equation~\ref{#1}}
\def\Eqref#1{Equation~(\ref{#1})}

\def\1{\bm{1}}

\DeclareMathAlphabet{\mathsfit}{\encodingdefault}{\sfdefault}{m}{sl}
\SetMathAlphabet{\mathsfit}{bold}{\encodingdefault}{\sfdefault}{bx}{n}

\DeclareMathOperator*{\argmin}{arg\,min}

\usepackage{amsmath}
\usepackage{amssymb}
\usepackage{mathtools}
\usepackage{amsthm}

\usepackage[utf8]{inputenc} 
\usepackage[T1]{fontenc}    
\usepackage{hyperref}       
\usepackage{url}            
\usepackage{booktabs}       
\usepackage{amsfonts}       
\usepackage{nicefrac}       
\usepackage{microtype}      
\usepackage{xcolor}         

\usepackage{graphicx}
\usepackage{times}
\usepackage{latexsym}
\usepackage{multirow}
\usepackage{balance}
\usepackage{bbm}
\usepackage{paralist}
\usepackage{makecell}
\usepackage[labelformat=simple]{subcaption}

\usepackage{xspace}
\usepackage{multicol}
\usepackage{transparent}
\usepackage{array}
\usepackage{bm}
\usepackage{tabularx}
\usepackage{setspace}
\usepackage{pifont}
\usepackage{thmtools, thm-restate}
\usepackage{wrapfig}

\newcommand{\eg}{\textit{e.g.}}
\newcommand{\ie}{\emph{i.e.}}
\newcommand{\method}{MAE-LM\xspace}
\newcommand{\bs}[1]{\boldsymbol{#1}}
\newcommand{\mask}{\texttt{[MASK]}\xspace}
\newcommand{\rank}{\text{rank}\xspace}

\newcommand\tp[2][-1]{{#2}^{\mkern#1mu\top}}

\theoremstyle{plain}
\newtheorem{theorem}{Theorem}[section]

\newtheorem{lemma}[theorem]{Lemma}

\theoremstyle{definition}

\theoremstyle{remark}
\newtheorem*{remark}{Remark}

\title{Representation Deficiency in Masked Language Modeling}

\author{
    Yu Meng${}^1$\thanks{Work done during internship at Meta AI.} \quad Jitin Krishnan${}^2$ \quad Sinong Wang${}^2$ \quad Qifan Wang${}^2$ \quad Yuning Mao${}^2$ \quad \\
    \textbf{ Han Fang${}^2$ \quad Marjan Ghazvininejad${}^2$ \quad Jiawei Han${}^1$ \quad Luke Zettlemoyer${}^2$}\\
    ${}^1$University of Illinois Urbana-Champaign \ \ \ ${}^2$Meta AI \\
    ${}^1$\texttt{\{yumeng5, hanj\}@illinois.edu} \ \ \
    ${}^2$\texttt{\{jitinkrishnan, sinongwang, } \\
    \texttt{ wqfcr, yuningm, hanfang, ghazvini, lsz\}@meta.com
    }
}

\iclrfinalcopy

\begin{document}

\maketitle

\begin{abstract}
Masked Language Modeling (MLM) has been one of the most prominent approaches for pretraining bidirectional text encoders due to its simplicity and effectiveness.
One notable concern about MLM is that the special \mask symbol causes a discrepancy between pretraining data and downstream data as it is present only in pretraining but not in fine-tuning.
In this work, we offer a new perspective on the consequence of such a discrepancy:
We demonstrate empirically and theoretically that MLM pretraining allocates some model dimensions exclusively for representing \mask tokens, resulting in a representation deficiency for real tokens and limiting the pretrained model's expressiveness when it is adapted to downstream data without \mask tokens. 
Motivated by the identified issue, we propose \method, which pretrains the Masked Autoencoder architecture with MLM where \mask tokens are excluded from the encoder.
Empirically, we show that \method improves the utilization of model dimensions for real token representations, and \method consistently outperforms MLM-pretrained models on the GLUE and SQuAD benchmarks.
\end{abstract}

\section{Introduction}

Pretraining text encoders to learn from bidirectional contexts has achieved enormous success in various natural language processing (NLP) tasks~\citep{clark2020electra,devlin2019bert,liu2019roberta}.
Masked Language Modeling (MLM)~\citep{devlin2019bert} is among one of the most prominent pretraining approaches due to its conceptual simplicity and empirical effectiveness:
By randomly masking a portion of input tokens and training a Transformer encoder to predict the original content based on the remaining bidirectional contexts, the model learns robust representations that generalize well to diverse downstream tasks.
Besides its broad impact in NLP, MLM has also been widely adopted for pretraining in other domains, such as images~\citep{Bao2021BEiTBP,Xie2022SimMIMAS}, videos~\citep{Tong2022VideoMAEMA,Wang2022BEVTBP} and graphs~\citep{Hou2022GraphMAESM}.

Despite its remarkable success, the effectiveness of MLM may be hindered by a discrepancy between pretraining and fine-tuning: 
The special \mask token occurs only in pretraining but not in downstream tasks.
While a few previous studies~\citep{clark2020electra,yang2019xlnet} have attempted to address this issue, they end up proposing new training objectives instead of systematically investigating why and how such a discrepancy impacts the generalization of MLM-pretrained models.

In this work, we study the consequence of including \mask tokens in MLM pretraining by examining the learned token representation space.
We empirically and theoretically show that \mask token representations exclusively occupy some model dimensions, thereby reducing the model capacity for representing real tokens.
Such a representation deficiency issue may not be simply addressed by fine-tuning on downstream tasks:
Those dimensions exclusively used for \mask tokens have not been pretrained to represent real tokens, and will have to be either trained from scratch on downstream data, raising the risk of overfitting~\citep{Hendrycks2019UsingPC,Kumar2022FineTuningCD}, or become unused, resulting in a waste of model capacity.

To address the representation deficiency issue, we propose a simple text encoder pretraining method, \method, which conducts MLM pretraining based on the Masked Autoencoder architecture~\citep{He2022MaskedAA}.
Notably, \mask tokens are omitted from the encoder's input so that the real token representations can utilize the entire model dimensions theoretically.
An auxiliary decoder, used only in pretraining and not in fine-tuning, takes the encoder's output representations and \mask positions to predict the original tokens.
We demonstrate empirically that by excluding \mask tokens from the encoder, \method improves the utilization of model dimensions both in pretraining and downstream tasks and achieves consistent and notable improvements over previous models pretrained by MLM and its variants on the GLUE and SQuAD benchmarks.\footnote{Code can be found at \url{https://github.com/yumeng5/MAE-LM}.}

Our main contributions are as follows: 
(1) We investigate the token representation space trained by MLM, and identify a previously unknown representation deficiency issue when the pretrained model is applied to real data without \mask tokens.
(2) Based on empirical and theoretical analyses, we explain why the representation deficiency issue occurs in the conventional MLM pretraining setup.
(3) We show that a simple pretraining method \method can address the identified issue and improve the downstream task performance of previous MLM-pretrained models under multiple pretraining and fine-tuning settings.

\section{Analysis of Token Representations in MLM}
\label{sec:analysis}
\subsection{Preliminaries}

\textbf{Transformer Encoder.}
Transformer encoders contain multiple Transformer layers, where each layer consists of two submodules, multi-head self-attention (MHSA) and feed-forward network (FFN). 
The self-attention mechanism uses queries $\bs{Q}$ and keys $\bs{K}$ to compute attention weights, and outputs a weighted sum of the values $\bs{V}$.
MHSA performs self-attention in parallel over $N$ heads as follows:
\begin{align*}
\text{Attn}(\bs{Q}, \bs{K}, \bs{V}) &= \text{Softmax}\left( \frac{\bs{Q} \bs{K}^\top }{\sqrt{d_h}}\right) \bs{V}, \\
\text{MHSA}(\bs{X}) = \text{Concat}(\text{head}_1, \dots, \text{head}_N) \bs{W}^O &, \quad 
\text{head}_h = \text{Attn}(\bs{XW}^Q_h, \bs{XW}^K_h, \bs{XW}^V_h),
\end{align*}
where $\bs{X} \in \mathbb{R}^{n\times d}$ is the input representations to MHSA, $n$ is the number of tokens and $d$ is the model dimension. 
$d_h$ is the dimension of head $h$ and is usually set to $d/N$.
$\bs{W}^Q_h, \bs{W}^K_h, \bs{W}^V_h \in \mathbb{R}^{d\times d_h}$ and $\bs{W}^O \in \mathbb{R}^{d\times d}$ are learnable weight matrices. 
The outputs of MHSA are further passed to FFN which learns nonlinear transformations to derive the final outputs of the Transformer layer.

\textbf{Masked Language Modeling (MLM).}
Given a text sequence $\bs{x} = [x_1, \dots, x_i, \dots, x_n]$, MLM randomly replaces a set of token positions $\mathcal{M}$ with \mask symbols. 
The resulting partially masked sequence $\hat{\bs{x}} = [x_1, \dots, \mask_i, \dots, x_n]$ is then fed to the Transformer encoder $\bs{\theta}$ which outputs the token representations $\bs{H} = [\bs{h}_1, \dots, \bs{h}_i, \dots, \bs{h}_n]$.
The encoder $\bs{\theta}$ is trained to predict the original token out of the vocabulary $\mathcal{V}$ at each masked position by minimizing the cross-entropy loss $\mathcal{L}_\text{MLM}$:
\begin{align}
\begin{split}
\label{eq:softmax}
p_{\bs{\theta}} (x_i | \hat{\bs{x}}) = \frac{\exp(\bs{e}_{x_i}^\top \bs{h}_i)}{\sum_{x' \in \mathcal{V}} \exp(\bs{e}_{x'}^\top \bs{h}_i)}, \quad \mathcal{L}_\text{MLM} = \mathbb{E} \left( - \sum_{i\in \mathcal{M}} \log  p_{\bs{\theta}} \left( x_i \big| \hat{\bs{x}} \right) \right),
\end{split}
\end{align}
where $\bs{e}_x$ refers to the embedding of token $x$.

\subsection{Rank-Deficient Real Token Representations}
MLM pretraining introduces a special \mask token to replace the token positions to be predicted, but such \mask tokens are usually absent from downstream task data.
Therefore, to study the PLM's capacity for downstream data representation, we examine the \emph{real token} representation space trained with MLM.
A common measure of the representation space capacity is the \emph{rank} of the data representation matrix~\citep{Ansuini2019IntrinsicDO,Bhojanapalli2020LowRankBI}.
In our case, this refers to the real token representation matrix $\bs{H}_{\mathcal{R}} \in \mathbb{R}^{n \times d}$ ($n \gg d$) where each row corresponds to the representation of a real token.
Ideally, one would hope $\bs{H}_{\mathcal{R}}$ to have high column rank (\ie, $\rank(\bs{H}_{\mathcal{R}}) \approx d$) so that more model dimensions are effective for modeling real tokens.
However, as we will show next, a portion of the model dimensions will be exclusively used for \mask token representations in MLM pretraining, so that $\bs{H}_{\mathcal{R}}$ is necessarily rank-deficient (\ie, not all model dimensions are leveraged to represent real tokens).



\begin{wrapfigure}{r}{0.57\textwidth}
\centering
\begin{subfigure}[t]{0.276\columnwidth}
\centering
\includegraphics[width=\linewidth]{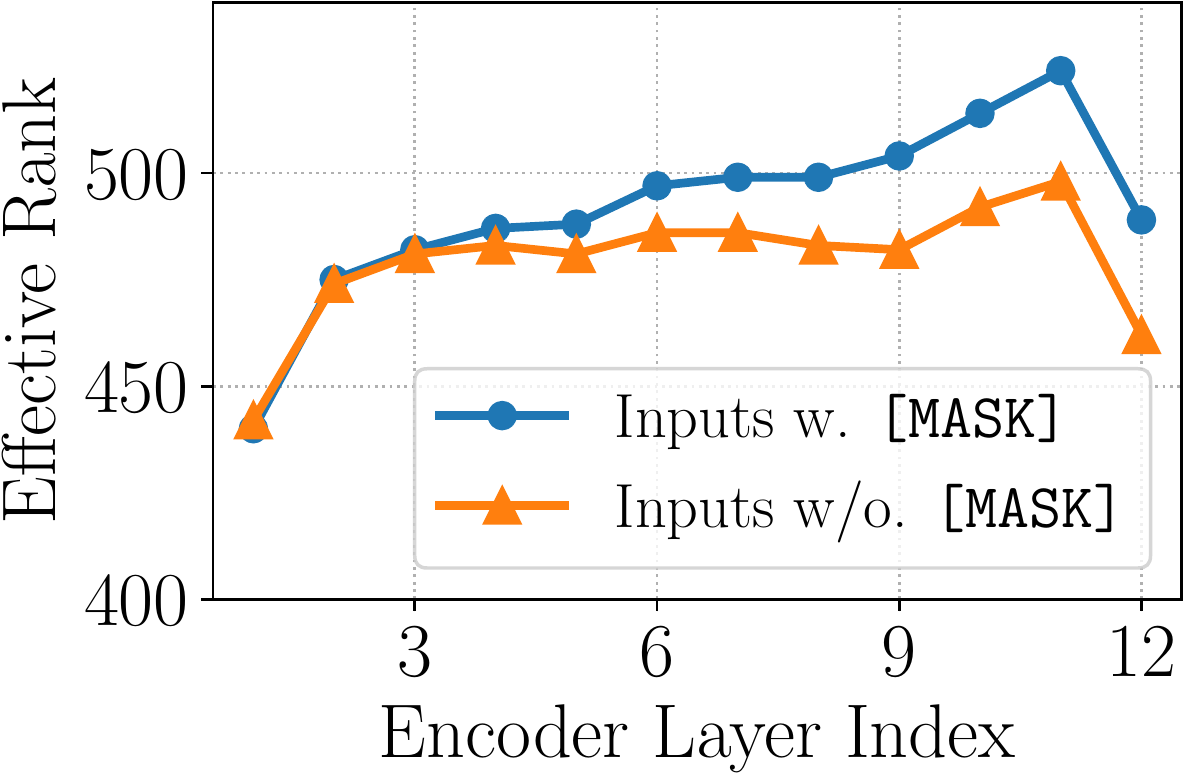}
\caption{}
\label{fig:pretrain_finetune_dim}
\end{subfigure}
~
\centering
\begin{subfigure}[t]{0.276\columnwidth}
\centering
\includegraphics[width=\linewidth]{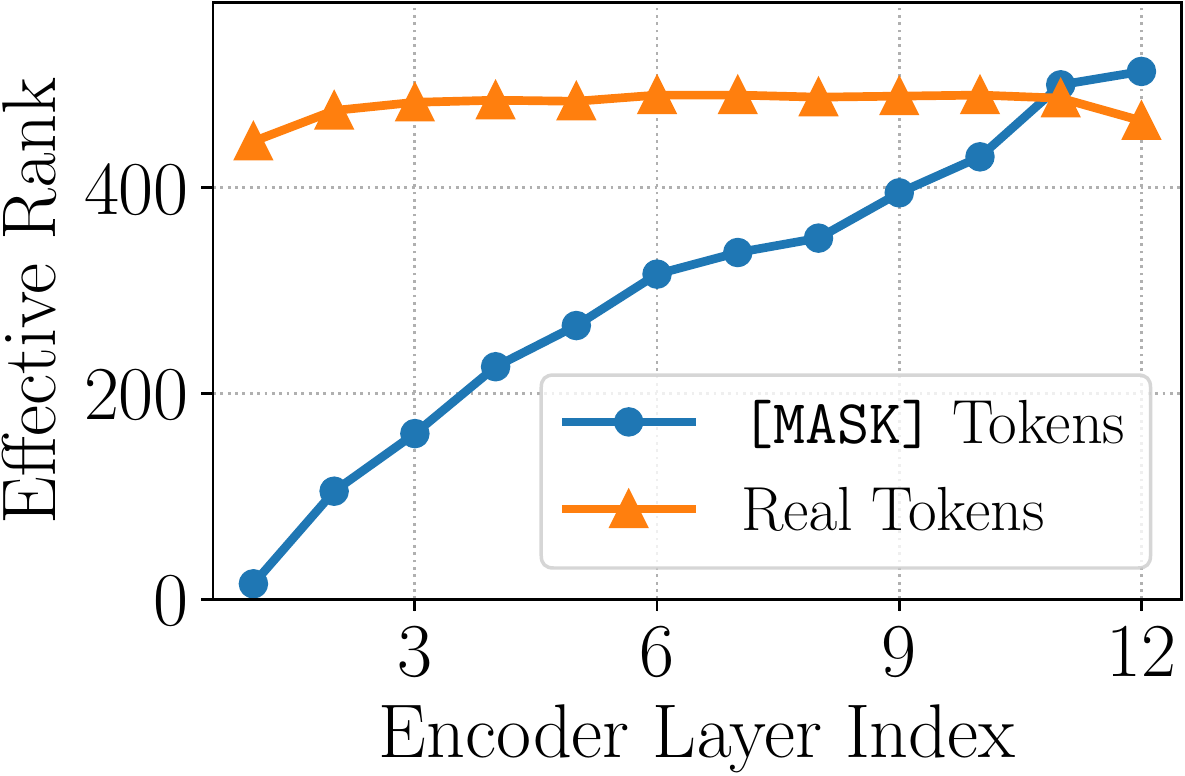}
\caption{}
\label{fig:mask_real_dim}
\end{subfigure}
\caption{In an MLM-pretrained model, (a) some model dimensions are exclusively used for representing \mask tokens, resulting in a representation deficiency for modeling inputs without \mask, especially in deeper layers; (b) the effective rank of \mask token representation space increases throughout Transformer layers.}
\vspace{-.5em}
\end{wrapfigure}

\textbf{Empirical Evidence.}
We evaluate the representation space of a pretrained $12$-layer RoBERTa$_\text{base}$ model~\citep{liu2019roberta} on the validation set of the pretraining corpus with $5$ million tokens.
We first apply $15\%$ random masks to these input sequences (same as the pretraining setting), and obtain the token representation matrix $\bs{H}^l \in \mathbb{R}^{n \times d}$ ($n \approx 5\times10^6$ is the total number of tokens in the corpus, $d = 768$ is the model dimension), which contains both real token and mask token representations, for each layer $l$ in the pretrained RoBERTa. 
We then feed the same input sequences in their original form (\ie, without \mask) to the pretrained RoBERTa model and obtain the token representation matrix $\widetilde{\bs{H}}^l \in \mathbb{R}^{n \times d}$ which consists of real token representations only.
Comparing the rank of $\widetilde{\bs{H}}^l$ with $\bs{H}^l$ gives insights about the change in representation capacity when adapting a pretrained MLM model to inputs without \mask.

Since numerical errors and small perturbations practically render any large matrix full-rank regardless of its actual rank, we compute the \emph{effective rank}~\citep{cai2020isotropy} of a matrix $\bs{H}$: We only consider $\bs{H}$'s most significant components  that account for the majority of the variance reflected by singular values. 
Given a threshold value $\tau$, we define the $\tau$-effective rank of $\bs{H}$ as 
$
\rank_\tau(\bs{H}) = \argmin_k \left( \frac{\sum_{i=1}^k \sigma_i^2}{\sum_{i=1}^d \sigma_i^2} \ge \tau \right),
$
where $\sigma_i$ is the $i$th largest singular value of $\bs{H}$.
For example, $\rank_{0.9}(\bs{H}) = 10$ means that $90\%$ of $\bs{H}$'s variance can be captured with $10$ dimensions.
We follow the definition of effective rank in \cite{cai2020isotropy} only to perform empirical computations of the rank to showcase the issue, and we do not use it in our theoretical analysis below.

Figure~\ref{fig:pretrain_finetune_dim} shows $\rank_{0.9}(\bs{H}^l)$ (Input w. \mask) and $\rank_{0.9}(\widetilde{\bs{H}}^l)$ (Input w/o. \mask).
It generally holds that $\rank_{0.9}(\widetilde{\bs{H}}^l) < \rank_{0.9}(\bs{H}^l)$, and the gap is more prominent in deeper layers. 
This demonstrates that some model dimensions are reserved for \mask token representations in almost all encoder layers, and these dimensions are not active when the input sequences consist of real tokens entirely.
Such representation deficiencies for modeling real tokens become more severe in deeper layers where \mask token representations occupy more dimensions, shown in Figure~\ref{fig:mask_real_dim}.

\textbf{Theoretical Analysis.}
We theoretically validate the empirical observation above that MLM necessarily allocates a subspace for \mask token representations which is not contained by the real token representation subspace, so that the real token representations are rank-deficient.

\begin{lemma}[Rank increase of \mask token representations in Transformer encoder]
\label{lem:rank}
The rank of \mask token representations will increase from the input layer to the output layer of an $L$-layer Transformer encoder trained with MLM (i.e., $\rank(\bs{H}_{\mathcal{M}}^L) \gg \rank(\bs{H}_{\mathcal{M}}^0)$).
\end{lemma}
\vspace{-1em}
\begin{proof} 
We first show that $\bs{H}_{\mathcal{M}}^L$ will be high-rank in a well-trained MLM model and then show that $\bs{H}_{\mathcal{M}}^0$ is necessarily low-rank, and thus the statement holds.

As shown in \Eqref{eq:softmax}, the output token probability distributions at masked positions are computed from the encoder's output representations $\bs{H}_{\mathcal{M}}^L \in \mathbb{R}^{m \times d}$ and token embeddings $\bs{E} \in \mathbb{R}^{|\mathcal{V}| \times d}$.
Denote the true log probability distributions of the masked token prediction task as $\bs{T} \in \mathbb{R}^{m \times |\mathcal{V}|}$:
\setlength\arraycolsep{3pt}
$$
\bs{T} = \begin{bmatrix}
\log p\left(x_1|\hat{\bs{x}}_1\right) & \log p\left(x_2|\hat{\bs{x}}_1\right) & \cdots & \log p\left(x_{|\mathcal{V}|}|\hat{\bs{x}}_1 \right) \\
\log p\left(x_1|\hat{\bs{x}}_2\right) & \log p\left(x_2|\hat{\bs{x}}_2\right) & \cdots & \log p\left(x_{|\mathcal{V}|}|\hat{\bs{x}}_2 \right) \\
\vdots & \vdots & \ddots & \vdots \\
\log p\left(x_1|\hat{\bs{x}}_m\right) & \log p\left(x_2|\hat{\bs{x}}_m\right) & \cdots & \log p\left(x_{|\mathcal{V}|}|\hat{\bs{x}}_m \right)
\end{bmatrix},
$$
then $\bs{H}_{\mathcal{M}}^L$ and $\bs{E}$ are trained to approximate $\bs{T}$ with a row shift (due to the softmax normalization)~\citep{Yang2018BreakingTS}:
\begin{equation}
\label{eq:mlm_matrix}
\bs{H}_{\mathcal{M}}^L \bs{E}^\top \approx \bs{T} + \bs{c}\bs{1}^\top, 
\end{equation}
where $\bs{c} \in \mathbb{R}^m$ contains the shifting constant added to each row, and $\bs{1} \in \mathbb{R}^{|\mathcal{V}|}$ is a vector of all ones.

It is shown in \cite{Yang2018BreakingTS} that the true probability distribution $\bs{T}$ is high-rank (as high as $|\mathcal{V}|$) due to the complexity of natural language. Since $\rank(\bs{H}_{\mathcal{M}}^L \bs{E}^\top) \le \min \{ \rank(\bs{H}_{\mathcal{M}}^L), \rank(\bs{E}) \}$, both $\bs{H}_{\mathcal{M}}^L$ and $\bs{E}$ need to be high-rank to achieve a good approximation of $\bs{T} + \bs{c}\bs{1}^\top$.

Next, we show $\bs{H}_{\mathcal{M}}^0$ is low-rank. $\bs{H}_{\mathcal{M}}^0$ is the sum of token embeddings and position embeddings at masked positions:
$$
\bs{H}_{\mathcal{M}}^0 = \bs{1} \bs{e}_{\mask}^\top + \bs{P},
$$
where $\bs{e}_{\mask} \in \mathbb{R}^d$ is the \mask token embedding, and $\bs{P} \in \mathbb{R}^{m \times d}$ is the position embeddings. 

Since we have 
$\rank(\bs{1} \bs{e}_{\mask}^\top + \bs{P}) \le \rank(\bs{1} \bs{e}_{\mask}^\top) + \rank(\bs{P}) = \rank(\bs{P}) + 1$,
we only need to show $\bs{P}$ is low-rank.
Previous studies~\citep{he2020deberta,Ke2021RethinkingPE} have identified that position embeddings $\bs{P}$ and token embeddings $\bs{E}$ encode disjoint information, and are learned in separate subspaces of $\mathbb{R}^{d}$. 
Therefore, $\rank(\bs{P}) \le d - \rank(\bs{E})$. We also showed that $\bs{E}$ must be high-rank to satisfy \Eqref{eq:mlm_matrix}, and thus $\bs{P}$ is necessarily low-rank. Finally, $\bs{H}_{\mathcal{M}}^0$ is also low-rank as $\rank(\bs{H}_{\mathcal{M}}^0) \le \rank(\bs{P}) + 1$.
\end{proof}

\begin{remark}
Lemma~\ref{lem:rank} corresponds to the empirical observation in Figure~\ref{fig:mask_real_dim}, and can be intuitively interpreted as a necessary consequence of the \mask token contextualization process in Transformers: The \mask representations at the input layer are context-free, and they need to aggregate contextual information from other tokens in the sequence for predicting the original word, resulting in an increase in the information content of \mask token representations.
We also note that the rank increase statement does not necessarily apply to real token representations.
This is because MLM does not directly train the real token representations (\eg, the training objective in \Eqref{eq:mlm_matrix} does not apply to real token positions\footnote{Some MLM training settings adopt a trick that keeps $10\%$ of \mask as original tokens and randomly replaces another $10\%$ of \mask with other tokens. Even with this trick, the training signals on real token representations are scarce. Furthermore, later studies~\citep{Wettig2022ShouldYM} report that this trick is not necessary---training exclusively on \mask positions performs well.}).
\end{remark}

Based on Lemma~\ref{lem:rank}, we proceed to prove that $\bs{H}_{\mathcal{M}}^l$ occupies a different subspace that is not contained by the subspace of $\bs{H}_{\mathcal{R}}^l$, resulting in deficient representations for real tokens.
In the following, we analyze the rank change induced by the \emph{self-attention} mechanism since it is the source of contextualization of \mask tokens, and the effectiveness of text encoders is typically attributed to the contextualized representations~\citep{Ethayarajh2019HowCA}.
While we do not account for MLPs and residual connections, our analysis validates that the rank deficiency is caused by the self-attention mechanism, and in practice, MLPs and residual connections do not prevent the issue from happening. 
\begin{restatable}[Rank deficiency of real token representations]{theorem}{deficient}
\label{thm:deficient}
There exists some layer $l$ in the Transformer encoder where the real token representation $\bs{H}_{\mathcal{R}}^l$ is rank-deficient. 
In particular, the row space of $\bs{H}_{\mathcal{R}}^l$ does not contain the row space of \mask token representation $\bs{H}_{\mathcal{M}}^l$.
\end{restatable}
\vspace{-1em}
\begin{proof} 
We provide a proof sketch below. Detailed proofs can be found in Appendix~\ref{app:proof}.
We prove the statement by contradiction: Suppose that the row space of $\bs{H}_{\mathcal{R}}^l \in \mathbb{R}^{n \times d}$ contains the row space of $\bs{H}_{\mathcal{M}}^l \in \mathbb{R}^{m \times d}$, then we can represent $\bs{H}_{\mathcal{M}}^{l}$ with $\bs{H}_{\mathcal{R}}^{l}$ via a linear combination weight matrix $\bs{U}$:
\begin{equation}
\label{eq:linear}
\bs{H}_{\mathcal{M}}^{l} = \bs{U} \bs{H}_{\mathcal{R}}^{l}, \quad \bs{U} \in \mathbb{R}^{m \times n}.
\end{equation}
We show that under this assumption, $\bs{H}_{\mathcal{M}}^l$ will converge exponentially (with $l$) to a rank-$1$ matrix, which contradicts with Lemma~\ref{lem:rank}.
To examine the matrix rank, we follow the definition of matrix residual $\bs{R}^{l}$ \citep{Dong2021AttentionIN} which measures the difference between $\bs{H}_{\mathcal{R}}^l$ and a rank-$1$ matrix:
$$
\bs{R}^{l} = \bs{H}_{\mathcal{R}}^{l} - \bs{1} \bs{h}^\top, \quad \bs{h} = \argmin_{\bs{x}} \left\| \bs{H}_{\mathcal{R}}^{l} - \bs{1} \bs{x}^\top \right\|.
$$
Based on the self-attention formula and the assumption in \Eqref{eq:linear}, we can derive a bound for the norm of $\bs{R}^{l}$ as a function of $\bs{R}^{l-1}$:
\begin{align*}
\left\| \bs{R}^{l} \right\|_{1,\infty} \le 4 \epsilon \left\| \bs{R}^{l-1} \right\|_{1,\infty}^3, \quad
\epsilon = \left\| \frac{\bs{W}^Q \tp{\bs{W}^K}}{\sqrt{d}} \right\|_1 \left\| \bs{W}^V \bs{W}^O \right\|_{1,\infty} \left\|  \bs{U} \right\|_{\infty} \left( 1 + \left\|  \bs{U} \right\|_{\infty} \right).
\end{align*}
where $\left\|\cdot \right\|_{1,\infty}$ denotes the geometric mean of $\ell_1$ and $\ell_\infty$ norm.
This shows that $\left\| \bs{R}^{l} \right\|_{1,\infty}$ converges exponentially with $l$ to zero, and thus $\bs{H}_{\mathcal{R}}^l$ converges exponentially with $l$ to a rank-$1$ matrix.
We also have $\rank(\bs{H}^l_{\mathcal{M}}) \le \rank(\bs{H}^l_{\mathcal{R}})$ as the row space of $\bs{H}_{\mathcal{M}}^l$ is contained by the row space of $\bs{H}_{\mathcal{R}}^l$.  Hence, $\bs{H}^l_{\mathcal{M}}$ will also converge exponentially to a rank-$1$ matrix, which contradicts with Lemma~\ref{lem:rank}.
Therefore, the statement holds.
\end{proof}
\begin{remark}
Theorem~\ref{thm:deficient} demonstrates that at least some \mask token representations and real token representations need to be linearly independent so that the rank of $\bs{H}_{\mathcal{M}}^l$ may increase through encoder layers.
As a result, the real token representation $\bs{H}_{\mathcal{R}}^l$ cannot utilize the entire model dimensions and is prone to rank deficiency.
\end{remark}

\section{\method: Masked Autoencoders for MLM}
\begin{wrapfigure}{r}{0.6\textwidth}
\centering
\includegraphics[width=0.6\columnwidth]{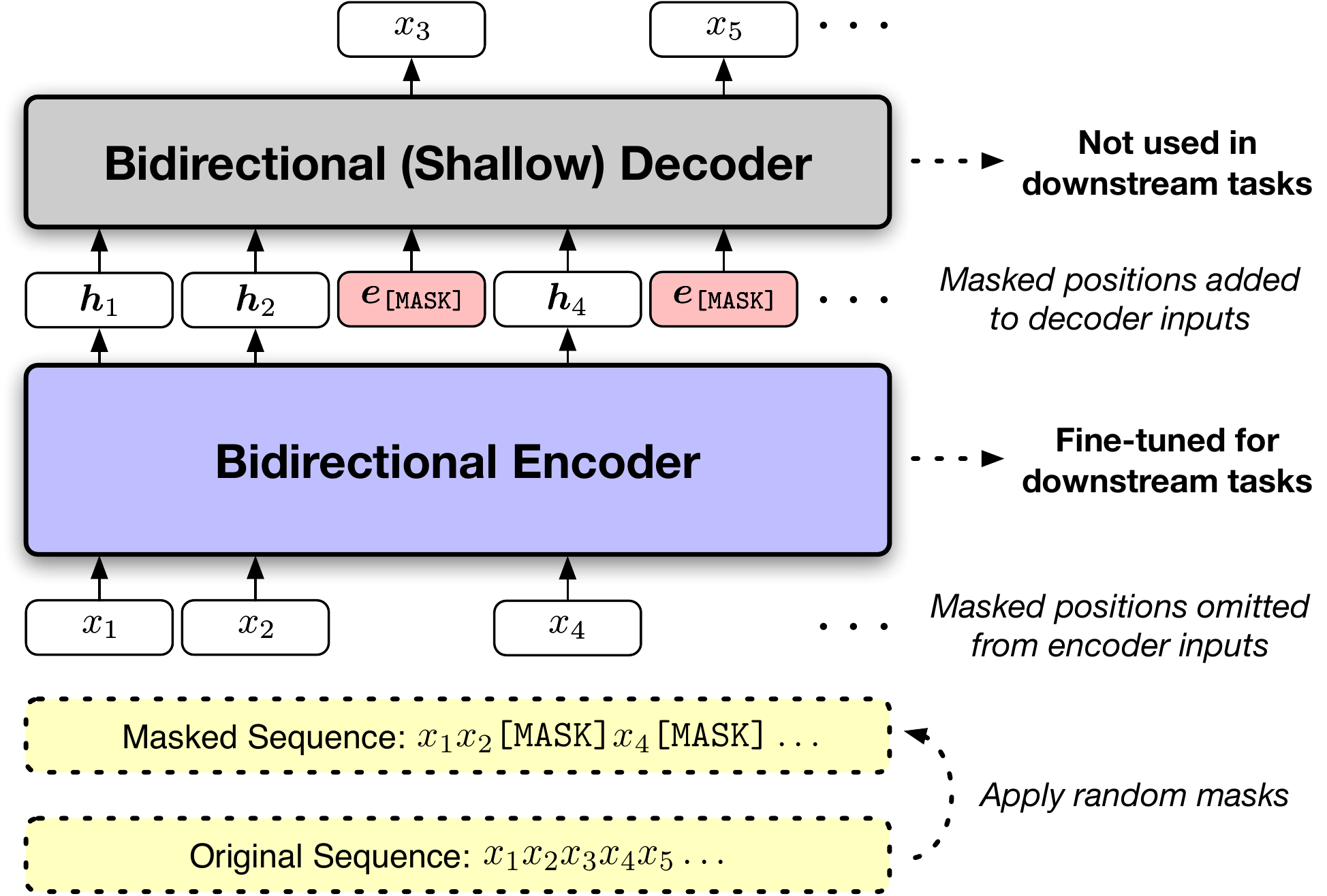}
\caption{Overview of \method. Masked positions are omitted from encoder inputs so that the encoder purely models real tokens. A shallow decoder takes the encoder's output representations and masked positions to predict the original tokens. After pretraining, only the encoder (but not the decoder) is fine-tuned for downstream tasks.}
\label{fig:mae_lm}
\vspace{-.5em}
\end{wrapfigure}

To address the representation deficiency issue in MLM, we propose a simple framework \method, which pretrains bidirectional Transformer encoders using the MLM objective, but based on the Masked Autoencoder~\citep{He2022MaskedAA,Liao2022MaskMA} structure.
An overview of \method is shown in Figure~\ref{fig:mae_lm}.
While previous applications of the architecture are mainly motivated by the efficiency benefit of reduced input sequence lengths, its effects on the learned tokens representations have not been thoroughly studied. 

\textbf{Excluding} \mask{} \textbf{from the Encoder.}
An important design in \method is that \mask tokens are excluded from the encoder inputs so that no model dimensions will be used to represent \mask tokens.
Hence, the representations of real tokens $\bs{H}_{\mathcal{R}}$ can theoretically utilize the entire space $\mathbb{R}^{d}$, which addresses the representation bottleneck in conventional MLM pretraining.
Specifically, given a masked sequence $\hat{\bs{x}} = [x_1, \dots, \mask_i, \dots, x_n]$ and let $\mathcal{M}$ denote the set of masked positions, the encoder's input sequence $\bs{H}^0$ consists of the sum of token embeddings $\bs{e}_{x_i}$ and position embeddings $\bs{p}_{i}$ at real token positions $i \notin \mathcal{M}$:
$$
\bs{H}^0 = \left\{ \bs{h}_i^0 \right\}_{i \notin \mathcal{M}}, \quad \bs{h}_i^0 = \bs{e}_{x_i} + \bs{p}_{i}.
$$

\textbf{Decoder Configuration.}
In order to predict the original tokens at masked positions, the encoder's output token representations $\bs{H}^L$ are further passed to an auxiliary bidirectional decoder.
While standard Transformer decoders perform unidirectional self-attention (and cross-attention to encoder outputs) for autoregressive decoding, our decoder performs bidirectional self-attention (same as the encoder). It is called a decoder as it takes encoded representations as input and outputs tokens.
The decoder's input sequence $\widehat{\bs{H}}^0$ needs to include the \mask token embedding $\bs{e}_{\mask}$ and position embeddings $\bs{p}_{i}$ so that the decoder is aware of the positions to be predicted:
$$
\widehat{\bs{H}}^0 = \left\{ \widehat{\bs{h}}_i^0 \right\}_{1 \le i \le n}, \quad \widehat{\bs{h}}_i^0 = \begin{cases}
\bs{e}_{\mask} + \bs{p}_{i} & i \in  \mathcal{M} \\
\bs{h}_{i}^L + \bs{p}_{i} & i \notin \mathcal{M}
\end{cases}.
$$
The decoder's output representations will be trained with the MLM objective shown in \Eqref{eq:softmax}.
Since the decoder includes \mask tokens, it is subject to the representation deficiency for modeling real tokens as analyzed in Section~\ref{sec:analysis}.
Therefore, the decoder is \emph{not} used in fine-tuning on downstream tasks.
The decoder is made to be shallow (the decoder depth is $1/6-1/3$ of the encoder in our experiments) not only for pretraining efficiency, but also to push the encoder to learn robust token representations---if the decoder is too strong, it alone may learn the MLM task well without requiring good encoder representations $\bs{H}^L$.

Despite using an additional decoder in pretraining, \method's pretraining time cost is roughly equal to that of conventional MLM pretraining (\eg, RoBERTa).
This is because the exclusion of \mask tokens from the encoder practically reduces its input sequence length (\eg, $15\%$ random masks shorten the encoder's input length by $15\%$), bringing down the encoder's computation cost.

\begin{table*}[t]
\centering
\small 
\caption{
Standard single-task, single-model fine-tuning results (medians over five random seeds) evaluated on GLUE and SQuAD 2.0 development sets. 
Results not available in prior research are marked with ``--''.  
We use Spearman correlation for STS, Matthews correlation for CoLA, and accuracy for the other tasks on GLUE.
The ``AVG'' column contains the averaged results across the eight GLUE tasks. 
All baseline results are taken from public reports unless marked with (Ours). 
}
\vspace{-0.5em}
\resizebox{\textwidth}{!}{
\begin{tabular}{l*{9}{l}ll}
\toprule
\multirow{2}{*}{\textbf{Model}} & \multicolumn{9}{c}{\textbf{GLUE (Single-Task)}} & \multicolumn{2}{c}{\textbf{SQuAD 2.0}} \\ 
\cmidrule(lr){2-10}\cmidrule(lr){11-12}
& \textbf{MNLI-(m/mm)} & \textbf{QQP} & \textbf{QNLI} & \textbf{SST-2} & \textbf{CoLA} & \textbf{RTE} & \textbf{MRPC} & \textbf{STS-B} & \textbf{AVG} &
\textbf{EM} & \textbf{F1}\\
\midrule
\multicolumn{12}{c}{\textit{base} setting: Pretrained on Wikipedia \& Book Corpus ($16$GB)}  \\ 
\midrule
BERT 
& 84.5/- & 91.3 & 91.7 & 93.2 & 58.9 & 68.6 & 87.3 & {89.5} & 83.1 &  73.7 & 76.3 \\

ALBERT 
& 81.6/- & -- & -- & 90.3 & -- & -- & -- & -- & -- & 77.1 & 80.0 \\



UniLMv2 
& 86.1/86.1 & -- & -- & 93.2 & -- & -- & -- & -- & -- & 80.9 & 83.6\\

TUPE 
& 86.2/86.2 & 91.3 & 92.2 & 93.3 & 63.6 & 73.6 & 89.9 & 89.2 & 84.9 & -- & --\\ 

RoBERTa 
& 84.7/- & -- & -- & 92.7 & -- & -- & -- & -- & -- & -- & 79.7\\ 

RoBERTa (Ours)
& 85.9/85.8 & \textbf{91.6} & 92.3 & 93.7 & \textbf{64.3} & 75.5 & 88.7 & 89.5 & 85.2 & 78.3 & 81.5 \\

\method
& \textbf{87.2/87.1} & \textbf{91.6} & \textbf{92.9} & \textbf{93.8} & 63.1 & \textbf{79.1} & \textbf{90.2} & \textbf{90.9} & \textbf{86.1} & \textbf{81.1} & \textbf{84.1}

\\ \midrule

\multicolumn{12}{c}{\textit{base++} setting: Pretrained on larger pretraining corpora ($160$GB)} \\
\midrule
ALBERT 
& 82.4/- & -- & -- & 92.8 & -- & -- & -- & -- & -- & 76.3 & 79.1 \\
RoBERTa 
& 87.6/- & \textbf{91.9} & 92.8 & 94.8 & 63.6 &  78.7 & 90.2 &  {91.2} & 86.4 & 80.5 & 83.7\\
UniLMv2 
& 88.5/- & 91.7 & 93.5 & \textbf{95.1} & 65.2 & 81.3 & \textbf{91.8}  & 91.0 & 87.1 & 83.3 & 86.1\\
\method  
& \textbf{89.1}/\textbf{89.1}
 & 91.7 & \textbf{93.8} & \textbf{95.1} & \textbf{65.9} & \textbf{85.2} & 90.2 & \textbf{91.6} & \textbf{87.8}
& \textbf{83.5} & \textbf{86.5}
 \\ 
\bottomrule
\end{tabular}
}
\vspace{-1.5em}
\label{tab:main_res}
\end{table*}

\section{Experiments}
\label{sec:exp}
\subsection{Pretraining and Evaluation Setup}

\textbf{Pretraining Settings.} 
We evaluate \method mainly under the base model scale for two pretraining settings: \textit{base} and \textit{base++}.
Both settings pretrain $12$-layer Transformers with $768$ model dimensions.
The \textit{base} setting uses $16$GB training corpus following BERT~\citep{devlin2019bert} while the \textit{base++} setting uses $160$GB training corpus following RoBERTa~\citep{liu2019roberta}.
The details can be found in Appendix~\ref{app:hyper}.
Additional results of larger model scales are presented in Appendix~\ref{app:more_eval}.
All settings use the MLM objective for pretraining without any sequence-level tasks.


\textbf{Downstream Tasks and Fine-Tuning.} 
We evaluate the pretrained models on the GLUE~\citep{wang2018glue} and SQuAD 2.0~\citep{Rajpurkar2018KnowWY} benchmarks.  
The details about GLUE tasks can be found in Appendix~\ref{app:glue}.
We adopt standard fine-tuning as in BERT~\citep{devlin2019bert} and RoBERTa~\citep{liu2019roberta}.
The hyperparameter search space for fine-tuning can be found in Appendix~\ref{app:hyper}.
All reported fine-tuning results are the medians of five random seeds on GLUE and SQuAD, following previous studies~\citep{liu2019roberta}.
Additional few-shot and zero-shot evaluation results are presented in Appendix~\ref{app:more_eval}.

\textbf{Baselines.} 
We compare with various baselines pretrained by MLM (and variants of MLM) under each setting, including BERT~\citep{devlin2019bert}, ALBERT~\citep{lan2019albert}, UniLMv2~\citep{unilmv2}, TUPE~\citep{Ke2021RethinkingPE}, and RoBERTa~\citep{liu2019roberta}.
The baseline results, unless marked by ``(Ours)'', are taken from the original papers.
To eliminate the performance difference due to implementation details and computation environment, we also pretrain and fine-tune RoBERTa (the most important baseline) under exactly the same \textit{base} pretraining setting with \method, which is denoted with ``RoBERTa (Ours)''.


\subsection{Overall Results}

Table~\ref{tab:main_res} shows the results under the two base model pretraining settings on the GLUE and SQuAD 2.0 benchmarks.
Overall, \method outperforms previous models pretrained by MLM and its variants. 
Notably, the gains of \method over RoBERTa (the standard MLM pretrained model) are quite consistent across tasks and pretraining settings. 

\begin{figure}[!t]
\centering
\begin{minipage}{0.485\textwidth}
\begin{subfigure}[t]{0.48\textwidth}
\centering
\includegraphics[width=\textwidth]{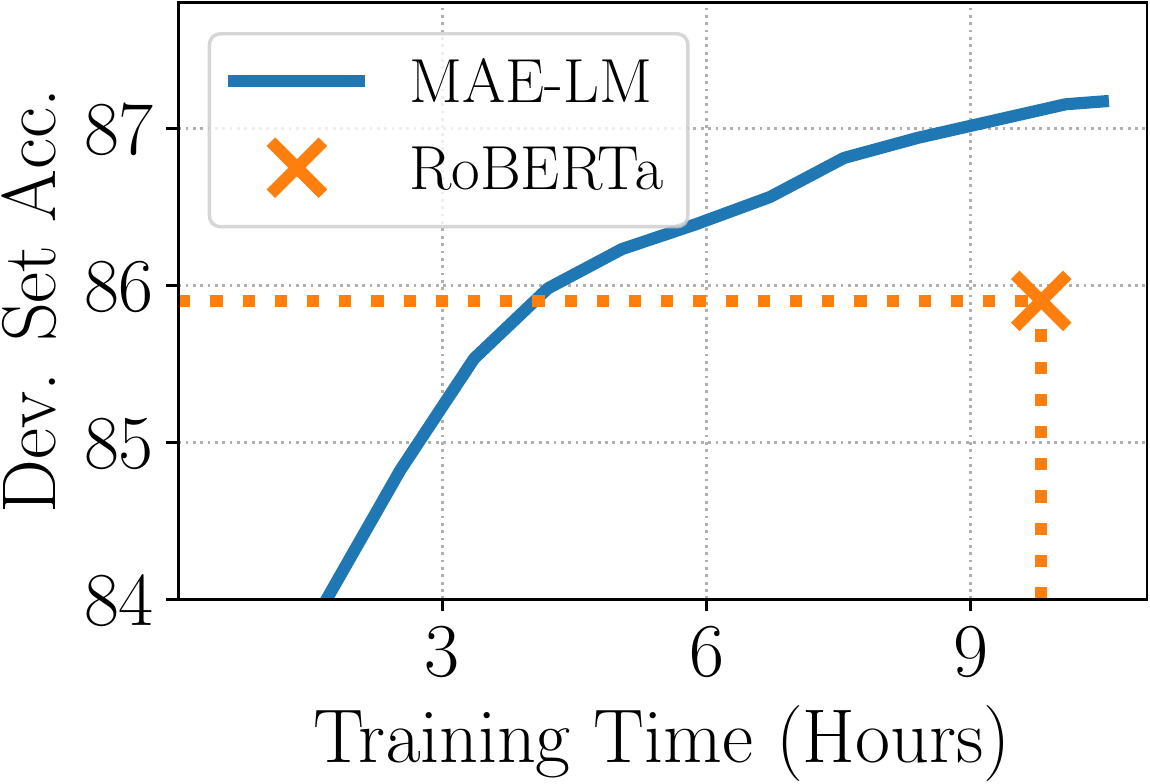}
\caption{MNLI-m}
\vspace{-.5em}
\end{subfigure}
~
\centering
\begin{subfigure}[t]{0.48\textwidth}
\centering
\includegraphics[width=\linewidth]{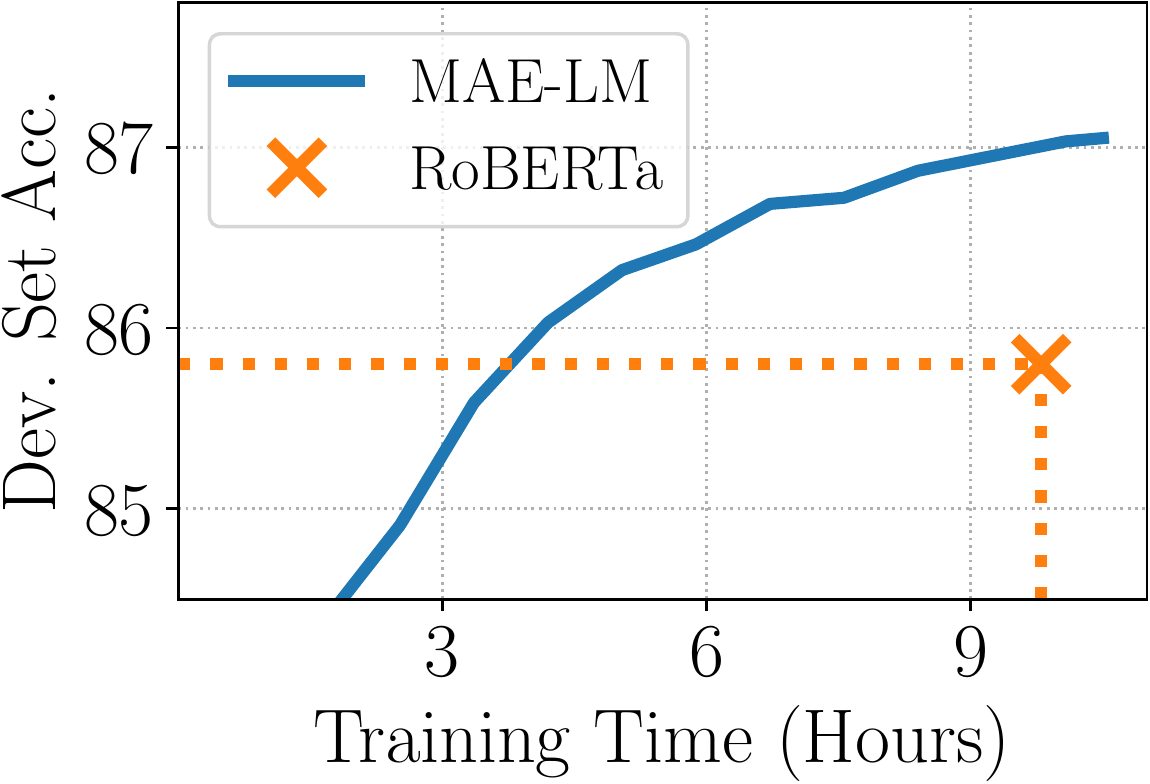}
\caption{MNLI-mm}
\vspace{-.5em}
\end{subfigure}
\caption{MNLI dev set accuracy by fine-tuning intermediate \method{}$_{\text{base}}$ checkpoints at different time steps.
We also mark the pretraining time and final performance of RoBERTa (Ours).
}
\label{fig:efficiency}
\end{minipage}
\hfill
\begin{minipage}{0.495\textwidth}
\begin{subfigure}[t]{0.49\textwidth}
\centering
\includegraphics[width=\textwidth]{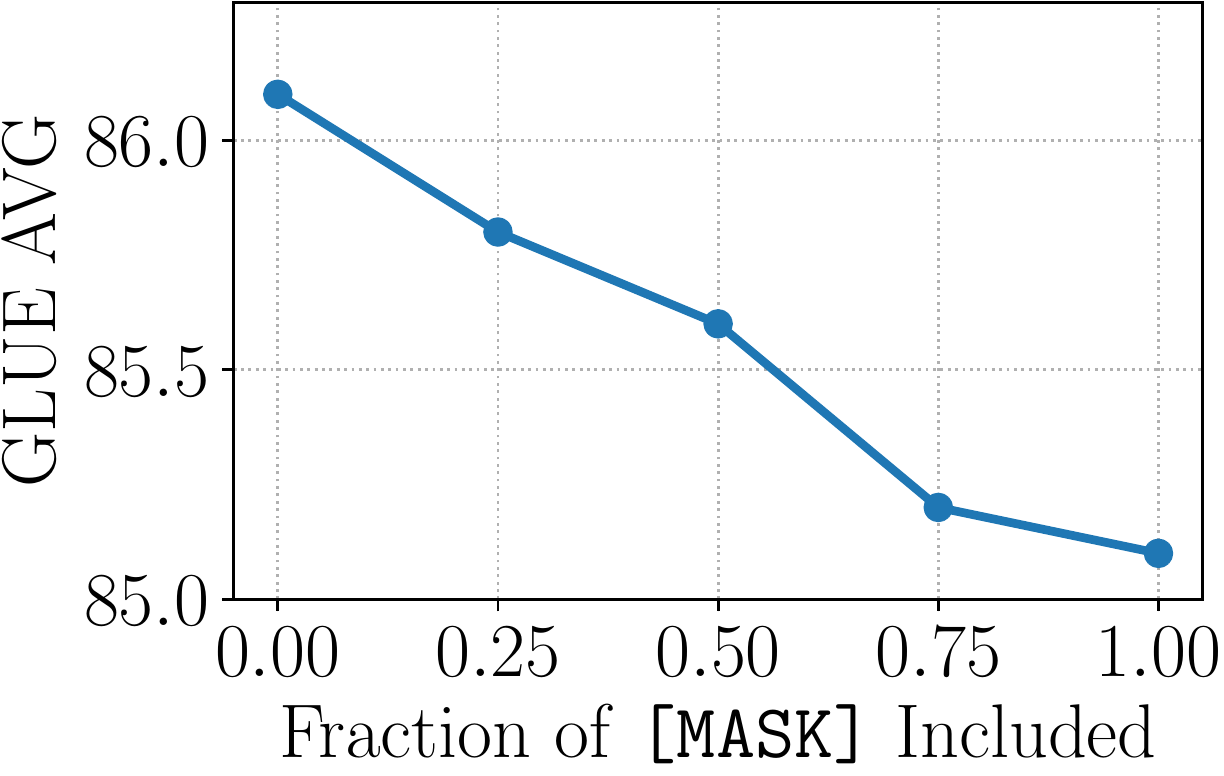}
\caption{GLUE}
\vspace{-.5em}
\end{subfigure}
~
\centering
\begin{subfigure}[t]{0.465\textwidth}
\centering
\includegraphics[width=\textwidth]{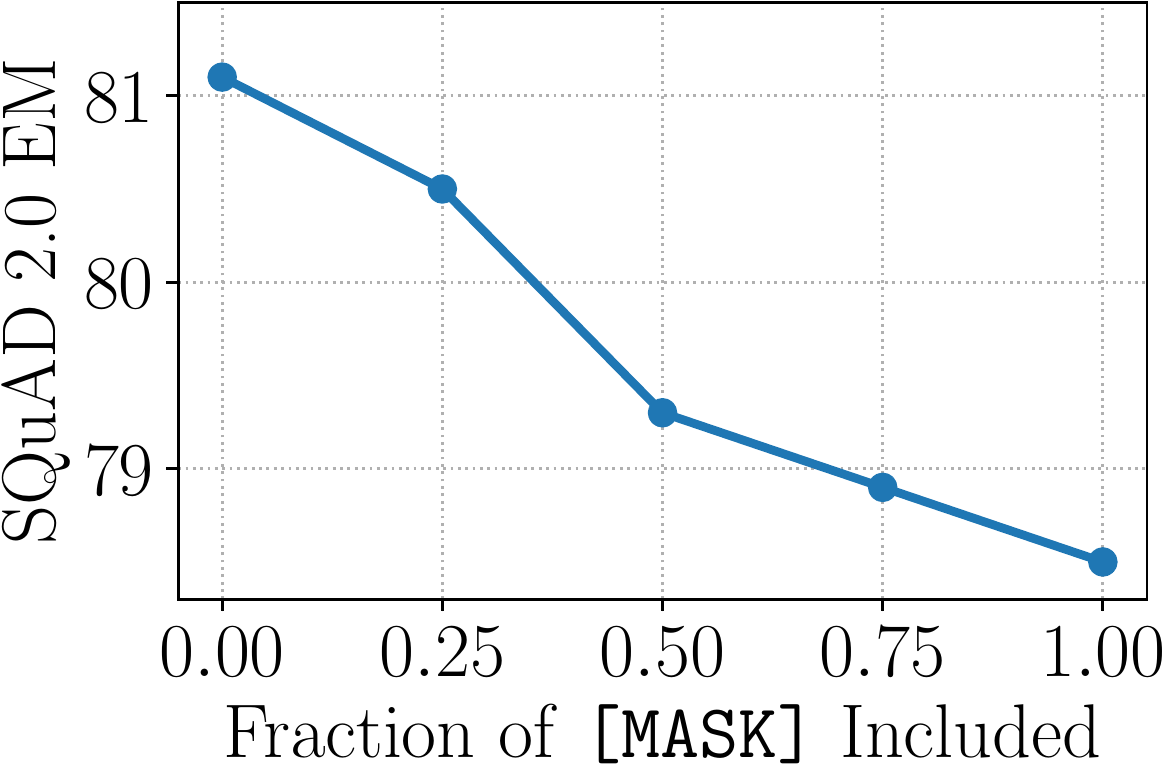}
\caption{SQuAD}
\vspace{-.5em}
\end{subfigure}
\vspace{.2em}
\caption{GLUE average scores and SQuAD EM scores when different fractions of \mask tokens are included in the input sequences to the encoder of \method{}$_{\text{base}}$. 
}
\label{fig:transition}
\end{minipage}
\vspace{-1em}
\end{figure}

\textbf{Pretraining Efficiency.} 
In Figure~\ref{fig:efficiency}, we illustrate \method{}$_{\text{base}}$'s fine-tuning performance when pretrained for different amounts of time.
\method takes slightly more time than RoBERTa when trained on the same amount of data, but to reach RoBERTa's MNLI accuracy, \method only needs about $40\%$ of its pretraining time.

\subsection{Ablation Studies}
\label{sec:ablation}
Table~\ref{tab:ablation} shows several groups of ablations to study the important components in \method.

\textbf{Naive Baselines.} 
To validate that the effectiveness of \method is not from simply using the additional decoder in pretraining, we first compare two naive baselines: (1) the standard MLM (enc. w. \mask) and (2) adding the same decoder used in \method but still pretrains the encoder with \mask tokens included in inputs (enc. w. \mask + dec.). The two baselines perform similarly, confirming that naively using the decoder does not benefit downstream tasks. 

\begin{wraptable}[24]{r}{0.6\textwidth}
\centering
\caption{
Ablations evaluated with GLUE average scores. 
The setting of \method{}$_\text{base}$ is: enc. w/o. \mask; aligned position encoding w. relative position encoding; bi. self-attention; $4$ layer, $768$ dimension.
}
\vspace{-.5em}
\small
\begin{tabular}{*{4}{l}}
\toprule
\textbf{Group} & \textbf{Setting} & \textbf{GLUE} \\
\midrule

\textbf{Original} & \method{}$_\text{base}$ & 86.1 \\ 
\midrule 

\textbf{Naive} & enc. w. \mask (\ie, MLM) & 85.2 \\
& enc. w. \mask + dec. & 85.1 \\

\midrule 
\textbf{Handling} & enc. w. \mask, dec. resets \mask & 85.9 \\
\mask & random replace w. real token &  85.1 \\
\midrule

\textbf{Position} & misaligned position encoding & 86.0 \\
\textbf{Encoding} & no relative position encoding & 86.1 \\
\midrule

\textbf{Decoder} &  bi. self-attention $+$ cross-attention & 85.4 \\
\textbf{Attention} & uni. self-attention $+$ cross-attention & 85.5 \\
& cross-attention & 86.0 \\
\midrule

\textbf{Decoder} & 2 layer, 768 dimension & 85.8 \\
\textbf{Size} & 6 layer, 768 dimension & 84.8 \\
& 4 layer, 512 dimension & 85.8 \\
& 4 layer, 1024 dimension & 85.5 \\
\bottomrule
\end{tabular}
\vspace{-.5em}
\label{tab:ablation} 
\end{wraptable}

\textbf{Handling} \mask{}\textbf{.}
We compare with other ways of handling \mask tokens in the encoder: 
(1) including \mask in encoder's inputs but resetting \mask token positions to the \mask token embedding $\bs{e}_{\mask}$ in decoder's inputs (enc. w. \mask, dec. resets \mask)
and (2) randomly replacing \mask tokens in encoder's inputs with other real tokens from the vocabulary (random replace w. real token).
The first variation improves the performance over vanilla MLM, showing that when \mask is present in the encoder, resetting the \mask token embeddings in the decoder helps.
This validates our analysis in Theorem~\ref{thm:deficient} that the rank increase of \mask token representations is the main cause of representation deficiency, and preventing \mask token representations in the encoder from being explicitly trained is one way to mitigate the issue, though it is slightly worse than completely excluding \mask from the encoder.
The second variation demonstrates that replacing \mask tokens with random real tokens, though avoiding the representation deficiency problem, worsens the context quality in pretraining.
On balance, it does not yield better results than MLM.

\textbf{Position Encoding.} 
\method aligns the position encoding based on each token's position in the original sequence, and the position indices of masked positions are skipped. 
\method also uses relative position encoding~\citep{raffel2019t5}.
We create two ablations: 
(1) apply consecutive position encoding that does not reflect the masked positions (misaligned position encoding); 
and (2) remove the relative position encoding from \method (no relative position encoding).
Overall, the variations in position encoding do not result in notable performance differences.

\textbf{Decoder Attention.} 
\method uses bidirectional self-attention in the decoder. 
We compare with other decoder attention configurations:
(1) additionally use cross-attention to encoder's output representations (bi. self-attention $+$ cross-attention);
(2) use unidirectional self-attention and cross-attention for autoregressive decoding of the entire sequence, similar to BART~\citep{Lewis2020BARTDS} (uni. self-attention $+$ cross-attention);
and (3) only use cross-attention (cross-attention).
Bidirectional self-attention only in the decoder is simple and performs the best.

\textbf{Decoder Size.} 
\method uses a $4$-layer decoder with the same dimensionality ($768$) as the encoder. 
We experiment with other decoder sizes (when the decoder's dimension is different from the encoder, we add a linear projection between the encoder's output and the decoder's input):
(1) $2$-layer, $768$ dimension;
(2) $6$-layer, $768$ dimension;
(3) $4$-layer, $512$ dimension;
and (4) $4$-layer, $1024$ dimension.
Overall, using a relatively small decoder yields good results.

\textbf{Gradual Transition from \method to Standard MLM.} 
To further examine the empirical benefits of excluding \mask tokens from \method's encoder, we create a set of ``stepping stones'' between \method and standard MLM as follows:
Out of all \mask tokens in the sequence $\hat{\bs{x}}$, we include a fraction ($\delta$) of them in the encoder's input sequence. 
The rest ($1-\delta$) of \mask tokens are excluded from the encoder's input and added to the decoder's input.
Then $\delta=0$ represents \method, and $\delta=1$ refers to the standard MLM\footnote{Although standard MLM (\ie, RoBERTa) does not have the decoder, its fine-tuning results are almost the same as $\delta=1$ (with the decoder) as shown in Table~\ref{tab:ablation}. 
}.
Figure~\ref{fig:transition} illustrates the fine-tuning performance changes on GLUE and SQuAD as we transition from \method to standard MLM.
There is a clear trend that including a higher portion of \mask tokens in the encoder degrades its performance.

\vspace{-.5em}
\subsection{\method Improves Model Dimension Utilization}

\begin{wrapfigure}{r}{0.57\textwidth}
\centering
\begin{subfigure}[t]{0.276\columnwidth}
\centering
\includegraphics[width=\linewidth]{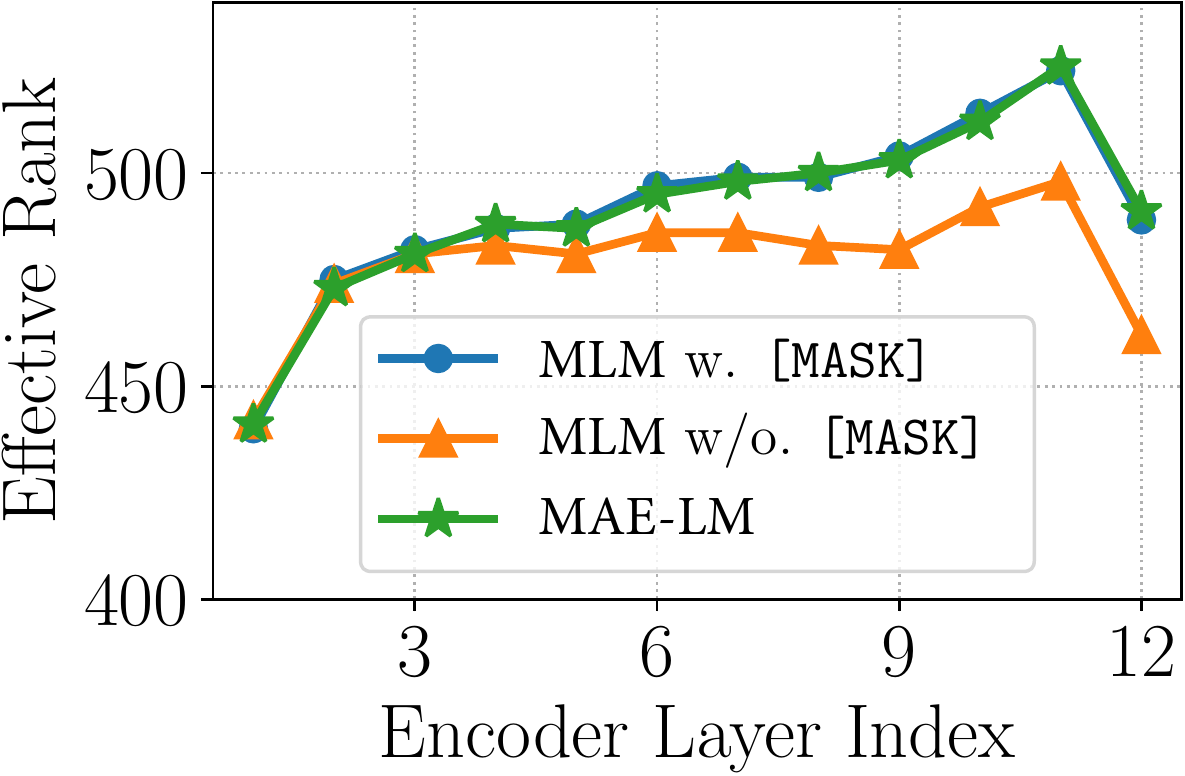}
\caption{}
\label{fig:eff_dim_pretrain_finetune_with_mae}
\end{subfigure}
~
\centering
\begin{subfigure}[t]{0.276\columnwidth}
\centering
\includegraphics[width=\linewidth]{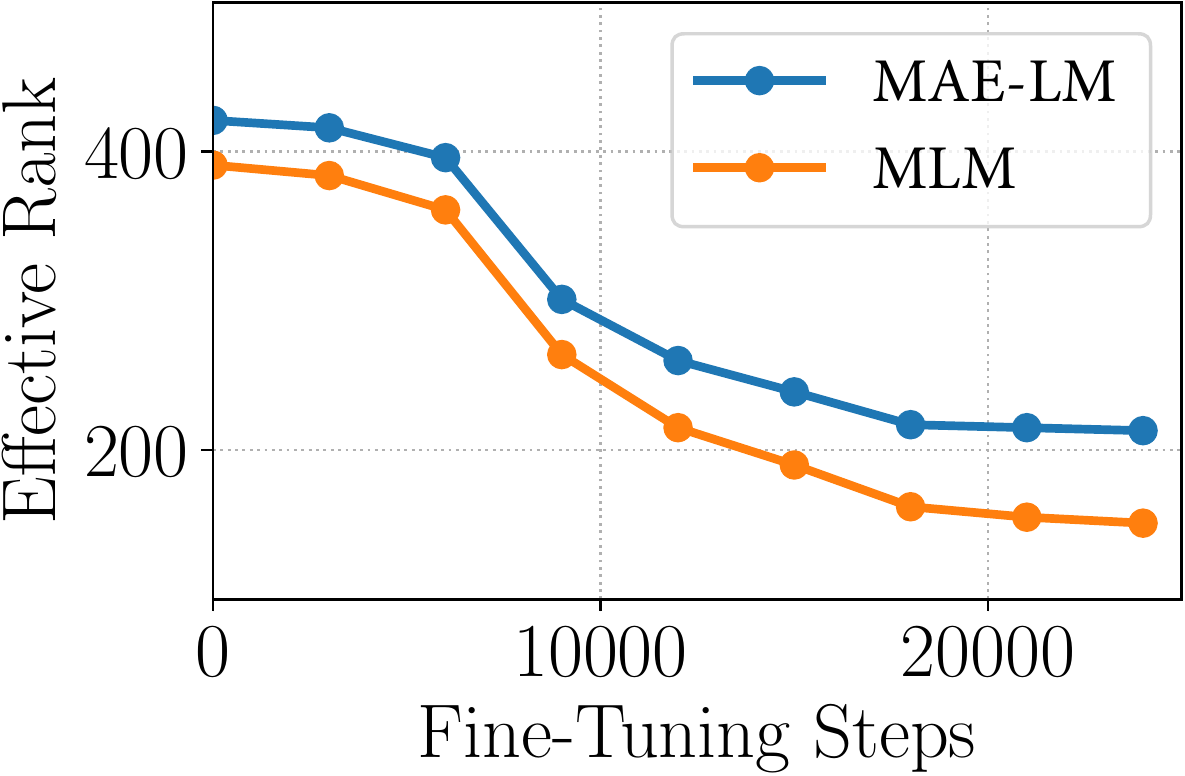}
\caption{}
\label{fig:eff_dim_finetune}
\end{subfigure}
\caption{(a) MAE-LM effectively closes the rank gap in vanilla MLM with inputs containing or not containing \texttt{[MASK]}. (b) During fine-tuning, the advantage in effective rank of \method over vanilla MLM still holds.}
\vspace{-.5em}
\end{wrapfigure}



To further validate the effectiveness of \method in improving the utilization of model dimensions for representing real tokens, we compute the $0.9$-effective rank of the encoder's token representations $\rank_{0.9}(\bs{H}^L)$ both after pretraining (evaluated on the validation set of the pretraining corpus) and after further fine-tuning on MNLI.
Figure~\ref{fig:eff_dim_pretrain_finetune_with_mae} shows the effective rank throughout encoder layers for (1) RoBERTa when the inputs contain \texttt{[MASK]} (MLM w. \texttt{[MASK]}); (2) RoBERTa when the inputs are all real tokens (MLM w/o. \texttt{[MASK]}); and (3) \method. MAE-LM closes the gap caused by \texttt{[MASK]} tokens in vanilla MLM pretraining.
Figure~\ref{fig:eff_dim_finetune} further validates that MAE-LM maintains its advantage in the effective rank of real token representations during fine-tuning on MNLI.
This highlights the importance of addressing the representation deficiency issue in pretraining: 
The model dimensions not pretrained to represent real tokens may not be easily leveraged in fine-tuning.

\section{Related Work}

\textbf{Language Model Pretraining.}
Various pretraining methods have been proposed for different purposes:
Standard autoregressive language modeling~\citep{Brown2020LanguageMA,radford2018improving,radford2019language} is commonly used to pretrain generative models that excel in text generation;
MLM~\citep{devlin2019bert,liu2019roberta} is prominently used to pretrain bidirectional text encoders to achieve superior performance for language understanding;
Other language modeling objectives~\citep{Lewis2020BARTDS,raffel2019t5} are designed to build sequence-to-sequence models that serve as both text generators and text encoders.
As one of the most prominent pretraining approaches, MLM has stimulated many follow-up developments for pretraining bidirectional encoders~\citep{unilmv2,clark2020electra,Gong2023ModelGeneratedPS,he2020deberta,Joshi2019SpanBERTIP,lan2019albert,Liao2022MaskMA,meng2021coco,Meng2022PretrainingTE,Sanh2019DistilBERTAD,yang2019xlnet}.
Remarkably, the idea of MLM is highly generalizable to different domains~\citep{Bao2021BEiTBP,Dosovitskiy2020AnII,Hou2022GraphMAESM,Tong2022VideoMAEMA,Wang2022BEVTBP,Xie2022SimMIMAS} and leads to developments of unified pretraining frameworks for different modalities~\citep{Baevski2022EfficientSL,Baevski2022data2vecAG}.
Given the broad impact of MLM, our analyses of representation deficiency in MLM may provide insights for future developments of pretraining algorithms in various fields.

\textbf{Study of Pretrained Models' Representations.}
The powerful language representations learned by pretrained models have driven a series of studies to understand how linguistic knowledge is acquired through pretraining.
Previous work studying the token representations in pretrained encoders has found that deeper layers generate more contextualized token representations~\citep{Ethayarajh2019HowCA}, and these representations encode syntax structures~\citep{Goldberg2019AssessingBS,Hewitt2019ASP} and fine-grained word senses~\citep{Coenen2019VisualizingAM}, offering supporting evidence for the effectiveness of pretrained models in downstream tasks.
The success of learning such linguistic patterns is usually attributed to the self-attention mechanism which automatically learns to extract useful features through pretraining~\citep{clark2019does}.
Furthermore, different types of linguistic information are shown to be represented in a hierarchical way from shallower to deeper layers, reflecting the traditional NLP pipeline~\citep{Tenney2019BERTRT,Tenney2019WhatDY}.
There have also been prior efforts that investigate the limitations of pretrained models' representations.
It has been revealed that the contextualized embedding space learned by pretrained models is generally anisotropic~\citep{cai2020isotropy,Li2020OnTS} and is subject to a degeneration problem that token representations tend to be distributed into a narrow cone~\citep{Gao2019RepresentationDP}.
\cite{Gong2019EfficientTO} identify that self-attention in Transformers tends to assign higher weights to local tokens as well as the starting token, which motivates the design of a progressive stacking algorithm for efficient pretraining.
In this work, we investigate a previously unknown issue regarding MLM-pretrained models' representations that hinders the model's expressiveness on input sequences without \mask tokens. 
Our findings contribute a new perspective to understanding the limitations of representations in pretrained models.

\section{Conclusion}


\textbf{Limitations.}
The focus of our work is on MLM and our analyses do not apply to other pretraining settings not using \mask tokens, and we discuss potential implications of our findings on autoregressive language models in Appendix~\ref{app:discuss}. 
While the current large language models are mostly autoregressive models, we believe that text encoder models still have important and wide applications in NLP, including but not limited to (1) Non-generation tasks. Many natural language understanding tasks do not have to be modeled autoregressively, for which encoder-only models are generally more parameter efficient and effective~\citep{Zhong2023CanCU}. (2) Retrieval-augmented text generation~\citep{Lewis2020RetrievalAugmentedGF}, which typically uses an encoder for retrieval to enhance the generator’s factualness. (3) Reward models in reinforcement learning from human feedback (RLHF) can use encoder models~\citep{Song2023RewardCI}.
Empirically, we mainly compare with models pretrained by MLM and its simple variants and do not include all state-of-the-art models, as they typically require integrating multiple pretraining strategies and/or architecture changes~\citep{He2021DeBERTaV3ID}.

\textbf{Conclusion.} 
In this work, we investigate the discrepancy caused by \mask tokens in MLM pretraining and demonstrate for the first time that this will necessarily result in real token representations being rank-deficient, thus limiting the model's expressiveness on real data without \mask. 
We propose a simple method \method that excludes \mask tokens from the encoder in pretraining to address the representation deficiency issue.
We empirically show that \method improves the utilization of model dimensions for representing real tokens in pretraining and downstream tasks.
\method consistently outperforms MLM-pretrained models on the GLUE and SQuAD benchmarks across multiple pretraining settings. 


\subsubsection*{Acknowledgments}
Research was supported in part by U.S. National Science Foundation IIS-19-56151, the Molecule Maker Lab Institute: An AI Research Institutes program supported by NSF under Award No. 2019897, and the Institute for Geospatial Understanding through an Integrative Discovery Environment (I-GUIDE) by NSF under Award No. 2118329. Yu Meng was supported by a Google PhD Fellowship.

\bibliography{ref}
\bibliographystyle{iclr2024_conference}

\newpage
\appendix

\section{Detailed Proofs}
\label{app:proof}

\deficient*

\begin{proof} 

We prove the statement by contradiction: We suppose that the row space of $\bs{H}_{\mathcal{R}}^l$ always contains the row space of $\bs{H}_{\mathcal{M}}^l$ in all layers $1 \le l \le L$, and we will show that under this assumption, $\bs{H}_{\mathcal{M}}^l$ will converge exponentially (with $l$) to a rank-$1$ matrix, which contradicts with Lemma~\ref{lem:rank}.
In the following, we assume \emph{single-head} self-attention is used, and the analysis can be easily generalized to the multi-head case.

The following proof extends \cite{Dong2021AttentionIN} by considering the representations of real tokens and mask tokens separately and following the residual norm analysis in \cite{Dong2021AttentionIN} to study the rank changes.

The self-attention module in the $l$th layer takes the previous layer representations $\bs{H}$ (the superscript $l-1$ is omitted for convenience) as input and derives the output representations $\bs{H}'$:
\begin{align*} 
\bs{H}' &= \text{Attn}\left(\bs{H}\bs{W}^Q, \bs{H}\bs{W}^K, \bs{H}\bs{W}^V\right) \bs{W}^O\\
&= \text{Softmax}\left( \frac{\bs{H}\bs{W}^Q \tp{\bs{W}^K}\tp{\bs{H}} }{\sqrt{d}}\right) \bs{H}\bs{W}^V \bs{W}^O \\
&= \bs{A} \bs{H}\bs{W}^{VO},
\end{align*}
where we denote the attention matrix computed from softmax as $\bs{A}$, and $\bs{W}^{VO} = \bs{W}^V \bs{W}^O$.

We study how the real token representations change (\ie, comparing $\bs{H}_{\mathcal{R}}'$ with $\bs{H}_{\mathcal{R}}$) through the self-attention module.
To facilitate easy analyses, we partition the input token representation matrix $\bs{H} \in \mathbb{R}^{(n+m) \times d}$ into blocks consisting of real token representations $\bs{H}_{\mathcal{R}} \in \mathbb{R}^{n \times d}$ and \mask token representations $\bs{H}_{\mathcal{M}} \in \mathbb{R}^{m \times d}$, and partition the attention matrix $\bs{A}_{\mathcal{R}}$ into blocks consisting of attention weights from real tokens to real tokens $\bs{A}_{\mathcal{R}:\mathcal{R}} \in \mathbb{R}^{n \times n}$ and from real tokens to \mask tokens $\bs{A}_{\mathcal{R}:\mathcal{M}} \in \mathbb{R}^{n \times m}$:
$$
\bs{H} = \begin{bmatrix}
\bs{H}_{\mathcal{R}} \\
\bs{H}_{\mathcal{M}}
\end{bmatrix}, \quad 
\bs{A}_{\mathcal{R}} = \begin{bmatrix}
\bs{A}_{\mathcal{R}:\mathcal{R}} & \bs{A}_{\mathcal{R}:\mathcal{M}}
\end{bmatrix}.
$$
We further denote
$$
\bs{S}_{\mathcal{R}:\mathcal{R}} = \exp \left[ \bs{H}_{\mathcal{R}} \bs{W}^{QK} \bs{H}_{\mathcal{R}}^\top \right],  \quad 
\bs{S}_{\mathcal{R}:\mathcal{M}} = \exp \left[ \bs{H}_{\mathcal{R}} \bs{W}^{QK} \bs{H}_{\mathcal{M}}^\top \right],  \quad 
\bs{Z} = \text{diag}(\bs{S}_{\mathcal{R}:\mathcal{R}}\bs{1} + \bs{S}_{\mathcal{R}:\mathcal{M}}\bs{1}),
$$
where $\exp [\cdot]$ denotes the element-wise exponential function, $\text{diag}(\cdot)$ constructs a diagnal matrix from a vector, $\bs{W}^{QK} = \bs{W}^Q \tp{\bs{W}^K}/\sqrt{d}$, and $\bs{1}$ is a vector of all ones.
Then
$$
\bs{A}_{\mathcal{R}:\mathcal{R}} = \bs{Z}^{-1} \bs{S}_{\mathcal{R}:\mathcal{R}}, \quad \bs{A}_{\mathcal{R}:\mathcal{M}} = \bs{Z}^{-1} \bs{S}_{\mathcal{R}:\mathcal{M}}.
$$
Based on the above notations, the output representations at real token positions $\bs{H}'_{\mathcal{R}}$ can be written as:
\begin{equation}
\label{eq:self_attn}
\bs{H}'_{\mathcal{R}} = \bs{A}_{\mathcal{R}} \bs{H} \bs{W}^{VO} = 
\begin{bmatrix}
\bs{A}_{\mathcal{R}:\mathcal{R}} & \bs{A}_{\mathcal{R}:\mathcal{M}}
\end{bmatrix}
\begin{bmatrix}
\bs{H}_{\mathcal{R}} \\
\bs{H}_{\mathcal{M}}
\end{bmatrix}
\bs{W}^{VO}
= \bs{Z}^{-1} \left( \bs{S}_{\mathcal{R}:\mathcal{R}} \bs{H}_{\mathcal{R}} + \bs{S}_{\mathcal{R}:\mathcal{M}} \bs{H}_{\mathcal{M}} \right) \bs{W}^{VO}.
\end{equation}

If the row space of $\bs{H}_{\mathcal{R}}$ contains the row space of $\bs{H}_{\mathcal{M}}$, each row of $\bs{H}_{\mathcal{M}}$ can be represented as a linear combination of the rows in $\bs{H}_{\mathcal{R}}$:
$$
\bs{H}_{\mathcal{M}} = \bs{U} \bs{H}_{\mathcal{R}},
$$
where $\bs{U} \in \mathbb{R}^{m \times n}$ is the linear combination weight matrix. 
We can rescale the vector norm of each row in $\bs{H}_{\mathcal{M}}$ so that $\bs{U}$ has a row sum of one (\ie, $\bs{U} \bs{1} = \bs{1}$).

To examine the rank of real token representations, we examine the change in matrix residual through Transformer layers, inspired by \cite{Dong2021AttentionIN}. 
Specifically, we define the following residual $\bs{R}$ which measures the difference between $\bs{H}_{\mathcal{R}}$ and a rank-$1$ matrix: 
$$
\bs{R} = \bs{H}_{\mathcal{R}} - \bs{1} \bs{h}^\top, \quad \bs{h} = \argmin_{\bs{x}} \left\| \bs{H}_{\mathcal{R}} - \bs{1} \bs{x}^\top \right\|.
$$
We aim to show that the norm of $\bs{R}$ converges exponentially (with layer depth) to zero, meaning that $\bs{H}_{\mathcal{R}}$ converges (with layer depth) to a rank-$1$ matrix.

By plugging $\bs{H}_{\mathcal{R}} = \bs{R} + \bs{1} \bs{h}^\top$ and $\bs{H}_{\mathcal{M}} = \bs{U} \bs{H}_{\mathcal{R}} = \bs{U} \bs{R} + \bs{U}\bs{1} \bs{h}^\top = \bs{U} \bs{R} + \bs{1} \bs{h}^\top$ into \Eqref{eq:self_attn}, we obtain
\begin{align}
\label{eq:H_R}
\bs{H}'_{\mathcal{R}} &= \bs{Z}^{-1} \left( \bs{S}_{\mathcal{R}:\mathcal{R}} \left(\bs{R} + \bs{1} \bs{h}^\top \right) + \bs{S}_{\mathcal{R}:\mathcal{M}} \left(\bs{U}\bs{R} + \bs{1} \bs{h}^\top \right) \right) \bs{W}^{VO} \nonumber \\
&= \left( \bs{Z}^{-1} \left( \bs{S}_{\mathcal{R}:\mathcal{R}} + \bs{S}_{\mathcal{R}:\mathcal{M}} \bs{U} \right) \bs{R} + \underbrace{\bs{Z}^{-1} \left( \bs{S}_{\mathcal{R}:\mathcal{R}}\bs{1} + \bs{S}_{\mathcal{R}:\mathcal{M}}\bs{1} \right)}_{=\bs{1}} \bs{h}^\top \right) \bs{W}^{VO} \nonumber \\
&= \bs{Z}^{-1} \left( \bs{S}_{\mathcal{R}:\mathcal{R}} + \bs{S}_{\mathcal{R}:\mathcal{M}} \bs{U} \right) \bs{R} \bs{W}^{VO} + \bs{1} \bs{h}^\top \bs{W}^{VO}.
\end{align}
Next we write out $\bs{S}_{\mathcal{R}:\mathcal{R}}$ and $\bs{S}_{\mathcal{R}:\mathcal{M}}$:
{\small
\begin{align*}
\bs{S}_{\mathcal{R}:\mathcal{R}} &= \exp \left[ \bs{H}_{\mathcal{R}} \bs{W}^{QK} \bs{H}_{\mathcal{R}}^\top \right] \\
&= \exp \left[ (\bs{R} + \bs{1} \bs{h}^\top) \bs{W}^{QK} (\bs{R} + \bs{1} \bs{h}^\top)^\top \right] \\
&= \exp \left[ \bs{R} \bs{W}^{QK} \bs{R}^\top + \bs{1}\bs{h}^\top\bs{W}^{QK}\bs{R}^\top + \left(\bs{R} \bs{W}^{QK} \bs{h} + \bs{1} \bs{h}^\top \bs{W}^{QK} \bs{h} \right) \bs{1}^\top \right] \\
&= \exp \left[ \underbrace{\bs{R} \bs{W}^{QK} \bs{R}^\top}_{=\bs{F}} \right] \odot \exp \left[ \bs{1}\underbrace{\bs{h}^\top\bs{W}^{QK}\bs{R}^\top}_{=\bs{g}^\top} \right] \odot \exp \left[ \underbrace{\left(\bs{R} \bs{W}^{QK} \bs{h} + \bs{1} \bs{h}^\top \bs{W}^{QK} \bs{h} \right)}_{=\bs{c}} \bs{1}^\top \right],
\end{align*}}and
{\small
\begin{align*}
\bs{S}_{\mathcal{R}:\mathcal{M}} &= \exp \left[ \bs{H}_{\mathcal{R}} \bs{W}^{QK} \bs{H}_{\mathcal{M}}^\top \right] \\
&= \exp \left[ (\bs{R} + \bs{1} \bs{h}^\top) \bs{W}^{QK} (\bs{U} \bs{R} + \bs{1} \bs{h}^\top)^\top \right] \\
&= \exp \left[ \bs{R} \bs{W}^{QK} \bs{R}^\top \bs{U}^\top + \bs{1}\bs{h}^\top\bs{W}^{QK}\bs{R}^\top \bs{U}^\top + \left(\bs{R} \bs{W}^{QK} \bs{h} + \bs{1} \bs{h}^\top \bs{W}^{QK} \bs{h} \right) \bs{1}^\top \right] \\
&= \exp \left[ \underbrace{\bs{R} \bs{W}^{QK} \bs{R}^\top \bs{U}^\top}_{=\bs{F}'} \right] \odot \exp \left[ \bs{1} \underbrace{\bs{h}^\top\bs{W}^{QK}\bs{R}^\top \bs{U}^\top}_{=\bs{g}'^\top} \right] \odot \exp \left[ \underbrace{\left(\bs{R} \bs{W}^{QK} \bs{h} + \bs{1} \bs{h}^\top \bs{W}^{QK} \bs{h} \right)}_{=\bs{c}} \bs{1}^\top \right],
\end{align*}
}where $\odot$ denotes the element-wise product.
Let $\bs{F} = \bs{R} \bs{W}^{QK} \bs{R}^\top$, $\bs{F}' = \bs{R} \bs{W}^{QK} \bs{R}^\top \bs{U}^\top$, $\bs{g}^\top = \bs{h}^\top\bs{W}^{QK}\bs{R}^\top$, $\bs{g}'^\top = \bs{h}^\top\bs{W}^{QK}\bs{R}^\top \bs{U}^\top$,
and $\bs{c} = \bs{R} \bs{W}^{QK} \bs{h} + \bs{1} \bs{h}^\top \bs{W}^{QK} \bs{h}$, we can further write out $\bs{Z}$:
\begin{align*}
\bs{Z} &= \text{diag}\left( \bs{S}_{\mathcal{R}:\mathcal{R}}\bs{1} + \bs{S}_{\mathcal{R}:\mathcal{M}}\bs{1} \right) \\
&= \text{diag}\left( \left(  \left( \exp \left[ \bs{F} \right] \odot \exp \left[ \bs{1} \bs{g}^\top \right] \right)\bs{1} + \left(\exp \left[ \bs{F}' \right] \odot \exp \left[ \bs{1} \bs{g}'^\top \right] \right)\bs{1}\right) \odot \exp \left[ \bs{c} \right]\right).
\end{align*}
Let $
\widetilde{\bs{F}} = \begin{bmatrix}
\bs{F} & \bs{F}'
\end{bmatrix}
$ be the augmented matrix by combining the columns of $\bs{F}$ and $\bs{F}'$, and let $\overline{\bs{f}}$ and $\underline{\bs{f}}$ denote the maximum and minimum element across each row of $\widetilde{\bs{F}}$, respectively:
$$
\overline{f}_i = \max_j \widetilde{F}_{ij}, \quad \underline{f}_i = \min_j \widetilde{F}_{ij}.
$$
Then we can derive a lower bound of each element in $\bs{Z}^{-1} \bs{S}_{\mathcal{R}:\mathcal{R}}$:
\begin{align*}
\left[ \bs{Z}^{-1} \bs{S}_{\mathcal{R}:\mathcal{R}} \right]_{ij} &= \frac{\exp(F_{ij}) \exp(g_{j}) \exp(c_i)}{\left(\sum_{j'} \exp(F_{ij'}) \exp(g_{j'}) + \sum_{j'} \exp(F'_{ij'}) \exp(g'_{j'})\right) \exp(c_i)}\\
&\ge \frac{\exp(F_{ij}) \exp(g_{j})}{\exp \left(\overline{f}_i\right) \left( \sum_{j'} \exp(g_{j'}) + \sum_{j'} \exp(g'_{j'}) \right)}\\
&= \exp \left(F_{ij}-\overline{f}_i \right) \frac{\exp(g_{j})}{\sum_{j'} \exp(g_{j'}) + \sum_{j'} \exp(g'_{j'})}.
\end{align*}
Similarly, we can derive an upper bound:
$$
\left[ \bs{Z}^{-1} \bs{S}_{\mathcal{R}:\mathcal{R}} \right]_{ij} \le \exp \left(F_{ij}-\underline{f}_i \right) \frac{\exp(g_{j})}{\sum_{j'} \exp(g_{j'}) + \sum_{j'} \exp(g'_{j'})}.
$$
Using the the Taylor expansion of $\exp$, we have 
$$
\exp \left(F_{ij}-\overline{f}_i \right) \ge 1 + F_{ij}-\overline{f}_i \ge 1 + \underline{f}_i - \overline{f}_i, \quad \exp \left(F_{ij} - \underline{f}_i \right) \le 1 + 2 \left( F_{ij} - \underline{f}_i\right) \le 1 + 2 \left( \overline{f}_i -\underline{f}_i\right).
$$
Therefore,
{\small
$$
(1 + \underline{f}_i - \overline{f}_i) \frac{\exp(g_{j})}{\sum_{j'} \exp(g_{j'}) + \sum_{j'} \exp(g'_{j'})} \le \left[ \bs{Z}^{-1} \bs{S}_{\mathcal{R}:\mathcal{R}} \right]_{ij} \le (1 + 2\overline{f}_i - 2\underline{f}_i) \frac{\exp(g_{j})}{\sum_{j'} \exp(g_{j'}) + \sum_{j'} \exp(g'_{j'})}.
$$
}

Denote $\bs{D} = \text{diag}\left(\overline{\bs{f}} - \underline{\bs{f}}\right)$ and $g_{+} = \exp \left[ \bs{g}^\top \right]\bs{1} + \exp \left[ \bs{g}'^\top \right] \bs{1}$, then the above bound can be expressed in matrix form as follows (the inequality between matrices holds element-wise):
\begin{equation}
\label{eq:S_RR}
\frac{1}{g_{+}}(\bs{I} - \bs{D}) \bs{1} \exp \left[\bs{g}^\top\right] \le \bs{Z}^{-1} \bs{S}_{\mathcal{R}:\mathcal{R}} \le \frac{1}{g_{+}} (\bs{I} + 2\bs{D}) \bs{1} \exp \left[\bs{g}^\top\right].
\end{equation}
An analogous derivation gives the bound of $\bs{Z}^{-1} \bs{S}_{\mathcal{R}:\mathcal{M}}$:
\begin{equation}
\label{eq:S_RM}
\frac{1}{g_{+}}(\bs{I} - \bs{D}) \bs{1} \exp \left[\bs{g}'^\top\right] \le  \bs{Z}^{-1} \bs{S}_{\mathcal{R}:\mathcal{M}} \le \frac{1}{g_{+}} (\bs{I} + 2\bs{D}) \bs{1} \exp \left[\bs{g}'^\top\right].
\end{equation}
Since the upper and lower bounds are in very similar forms, we will only focus on the upper bound in the derivations below.

Combining \Eqref{eq:S_RR} with \Eqref{eq:S_RM}, we have 
\begin{align}
\label{eq:ineq}
\bs{Z}^{-1} \left( \bs{S}_{\mathcal{R}:\mathcal{R}} + \bs{S}_{\mathcal{R}:\mathcal{M}} \bs{U} \right) &\le \bs{1} \left(\underbrace{ \frac{\exp \left[\bs{g}^\top\right] + \exp \left[\bs{g}'^\top\right]\bs{U}}{g_{+}}}_{=\bs{r}^\top}  \right) + 2\bs{D} \bs{1} \left(\underbrace{ \frac{\exp \left[\bs{g}^\top\right] + \exp \left[\bs{g}'^\top\right] \bs{U}}{g_{+}}}_{=\bs{r}^\top} \right) \nonumber \\
&= \bs{1} \bs{r}^\top + 2\bs{D} \bs{1} \bs{r}^\top
\end{align}
Plugging \Eqref{eq:ineq} into \Eqref{eq:H_R}, we have 
$$
\bs{H}'_{\mathcal{R}} \le \left(\bs{1} \bs{r}^\top + 2\bs{D} \bs{1} \bs{r}^\top \right) \bs{R} \bs{W}^{VO} + \bs{1} \bs{h}^\top \bs{W}^{VO} = \bs{1} \left( \underbrace{\bs{r}^\top\bs{R} \bs{W}^{VO} + \bs{h}^\top \bs{W}^{VO}}_{=\bs{h}'^\top} \right) + 2\bs{D} \bs{1} \bs{r}^\top \bs{R} \bs{W}^{VO}.
$$
Therefore,
$$
\bs{H}'_{\mathcal{R}} - \bs{1}\bs{h}'^\top \le 2\bs{D} \bs{1} \bs{r}^\top \bs{R} \bs{W}^{VO}.
$$
With a similar derivation, we have the following lower bound:
$$
\bs{H}'_{\mathcal{R}} - \bs{1}\bs{h}'^\top \ge -\bs{D} \bs{1} \bs{r}^\top \bs{R} \bs{W}^{VO}.
$$
Overall, we can bound the element-wise absolute values of $\bs{R}' = \bs{H}'_{\mathcal{R}} - \bs{1}\bs{h}'^\top$, which measure the distance between $\bs{H}'_{\mathcal{R}}$ and a rank-$1$ matrix:
$$
\left| R'_{ij} \right| = \left|\left[ \bs{H}'_{\mathcal{R}} - \bs{1}\bs{h}'^\top \right]_{ij} \right| \le \left| \left[2\bs{D} \bs{1} \bs{r}^\top \bs{R} \bs{W}^{VO}\right]_{ij} \right|.
$$
This allows us to further bound the norm of $\bs{R}'$. For $\ell_1$ norm, we have
\begin{align*}
\left\| \bs{R}' \right\|_{1} &\le \left\| 2\bs{D} \bs{1} \bs{r}^\top \bs{R} \bs{W}^{VO} \right\|_{1} \\
&\le 2 \left\| \bs{D} \bs{1} \right\|_{\infty} \left\| \bs{r}^\top \bs{R} \bs{W}^{VO} \right\|_1 & \text{Based on H\"older's inequality} \\
&\le 2 \left\| \bs{D} \bs{1} \right\|_{\infty} \left\| \bs{r}^\top \right\|_{1} \left\|  \bs{R} \right\|_1 \left\| \bs{W}^{VO} \right\|_1, & \text{Submultiplicativity of matrix norms} \\
\end{align*}
where 
\begin{align*}
\left\| \bs{D} \bs{1} \right\|_{\infty} &= \max_{i} \left| \overline{f}_i - \underline{f}_i \right| \\
&\le 2 \left\| \widetilde{\bs{F}} \right\|_{1} \\ 
&\le 2 \max \left\{ \left\| \bs{R} \bs{W}^{QK} \bs{R}^\top \right\|_{1},  \left\| \bs{R} \bs{W}^{QK} \bs{R}^\top \bs{U}^\top \right\|_{1} \right\} \\
&\le 2 \left\|  \bs{R} \right\|_1 \left\| \bs{W}^{QK} \right\|_1 \left\|  \bs{R} \right\|_{\infty} \max \left\{ 1, \left\|  \bs{U} \right\|_{\infty} \right\} \\
&\le 2 \left\|  \bs{R} \right\|_1 \left\| \bs{W}^{QK} \right\|_1 \left\|  \bs{R} \right\|_{\infty} \left\|  \bs{U} \right\|_{\infty}, & \text{$\left\|  \bs{U} \right\|_{\infty} \ge 1$ since $\bs{U} \bs{1} = \bs{1}$}
\end{align*}
and 
\begin{align*}
\left\| \bs{r}^\top \right\|_{1} &\le \left\| \bs{r}^\top \right\|_{\infty} \\
&= \left\| \frac{\exp \left[\bs{g}^\top\right] + \exp \left[\bs{g}'^\top\right] \bs{U}}{g_{+}} \right\|_{\infty} \\
&\le \left\| \frac{\exp \left[\bs{g}^\top\right]}{g_{+}} \right\|_{\infty} + \left\| \frac{\exp \left[\bs{g}'^\top\right] \bs{U}}{g_{+}}\right\|_{\infty}\\
&\le 1 + \left\|  \bs{U} \right\|_{\infty}.
\end{align*}
Therefore, we can bound the $\ell_1$ norm of $\left\| \bs{R}' \right\|_{1}$ as follows:
\begin{equation}
\label{eq:l1_bound}
\left\| \bs{R}' \right\|_{1} \le 4  \left\| \bs{W}^{QK} \right\|_1 \left\| \bs{W}^{VO} \right\|_1 \left\|  \bs{U} \right\|_{\infty} \left( 1 + \left\|  \bs{U} \right\|_{\infty} \right) \left\|  \bs{R} \right\|_1^2 \left\|  \bs{R} \right\|_{\infty}.
\end{equation}
Similarly, we can obtain the bound for the $\ell_{\infty}$ norm of $\left\| \bs{R}' \right\|_{1}$:
\begin{equation}
\label{eq:linf_bound}
\left\| \bs{R}' \right\|_{\infty} \le 4  \left\| \bs{W}^{QK} \right\|_1 \left\| \bs{W}^{VO} \right\|_{\infty} \left\|  \bs{U} \right\|_{\infty} \left( 1 + \left\|  \bs{U} \right\|_{\infty} \right) \left\|  \bs{R} \right\|_1 \left\|  \bs{R} \right\|_{\infty}^2.
\end{equation}
Denote the geometric mean of $\left\| \bs{R} \right\|_{1}$ and $\left\| \bs{R} \right\|_{\infty}$ as $\left\| \bs{R} \right\|_{1,\infty} = \sqrt{\left\| \bs{R} \right\|_{1}\left\| \bs{R} \right\|_{\infty}}$, then from \Eqref{eq:l1_bound} and \Eqref{eq:linf_bound}, we have
\begin{align*}
\left\| \bs{R}' \right\|_{1,\infty} &\le 4 \underbrace{\left\| \bs{W}^{QK} \right\|_1 \left\| \bs{W}^{VO} \right\|_{1,\infty} \left\|  \bs{U} \right\|_{\infty} \left( 1 + \left\|  \bs{U} \right\|_{\infty} \right)}_{=\epsilon} \left\| \bs{R} \right\|_{1,\infty}^3 \\
&= 4 \epsilon \left\| \bs{R} \right\|_{1,\infty}^3.
\end{align*}
The above inequality reflects how the residual changes within one self-attention layer. Applying it recursively throughout all layers in an $L$-layer encoder, we have:
$$
\left\| \bs{R}^L \right\|_{1,\infty} \le \left(4 \overline{\epsilon} \right)^{\frac{3^L-1}{2}} \left\| \bs{R}^0 \right\|_{1,\infty}^{3^L}, \quad \overline{\epsilon} = \max_{l} \epsilon^{l},
$$
where $\bs{R}^L$ and $\bs{R}^0$ denote the residuals corresponding to the encoder's output real token representations $\bs{H}_{\mathcal{R}}^L$ and input real token representations $\bs{H}_{\mathcal{R}}^0$, respectively.

This demonstrates that the residual norms of real token representations converge exponentially (with layer depth) to zero. 
Hence, the real token representation matrix $\bs{H}_{\mathcal{R}}^l$ converges exponentially (with layer depth) to a rank-$1$ matrix. 
Since the row space of \mask token representations $\bs{H}_{\mathcal{M}}^l$ is contained by the row space of $\bs{H}_{\mathcal{R}}^l$, we have $\rank(\bs{H}^l_{\mathcal{M}}) \le \rank(\bs{H}^l_{\mathcal{R}})$, and $\bs{H}^l_{\mathcal{M}}$ will also converge exponentially (with layer depth) to a rank-$1$ matrix, which contradicts with Lemma~\ref{lem:rank}.
Finally, we conclude that the row space of $\bs{H}_{\mathcal{R}}^l$ must not contain the row space of $\bs{H}_{\mathcal{M}}^l$, which necessarily implies that $\bs{H}_{\mathcal{R}}^l$ is rank-deficient.
\end{proof}

\section{Details about GLUE Tasks}
\label{app:glue}
More details of all the GLUE tasks can be found as follows.

\textbf{MNLI:} The Multi-genre Natural Language Inference~\citep{MNLI} task includes $393$K training examples from crowdsourcing. 
The goal is to predict if a premise sentence entails, contradicts, or is neutral with respect to a given hypothesis sentence. 

\textbf{QQP:} Question Pairs~\citep{QQP} includes $364$K training examples from the Quora question-answering website. 
The task is to determine if two given questions are semantically equivalent.

\textbf{QNLI:} Question Natural Language Inference includes $108$K training examples derived from the Stanford Question Answering Dataset (SQuAD)~\citep{Rajpurkar2018KnowWY}. 
The task is to predict if a sentence contains the answer to a given question.

\textbf{SST-2:} Stanford Sentiment Treebank~\citep{SST-2} includes $67$K training examples on movie reviews with human annotations. 
The task is to determine if a given sentence has positive or negative sentiment. 

\textbf{CoLA:} Corpus of Linguistic Acceptability~\citep{COLA} includes $8.5$K training examples from books and journal articles on linguistic theory. 
The task is to determine if a given sentence is linguistically acceptable. 

\textbf{RTE:} Recognizing Textual Entailment~\citep{RTE-5,RTE-1,RTE-2,RTE-3} includes $2.5$K training examples from textual entailment challenges. 
The task is to predict if a premise sentence entails a given hypothesis sentence.

\textbf{MRPC:} Microsoft Research Paraphrase Corpus~\citep{MRPC} includes $3.7$K training examples collected from news sources. 
The task is to predict if two given sentences are semantically equivalent.

\textbf{STS-B:} Semantic Textual Similarity~\citep{STS-B} includes $5.8$K training examples collected from multiple sources on sentence pair semantic similarity annotated by humans.
The task is to predict the semantic similarity of two sentences (based on a $1$ to $5$ scoring scale).

\section{Implementation Details}
\label{app:implementation} 

\textbf{Details of Pretraining Settings.}
The \textit{base} setting follows BERT$_\text{base}$~\citep{devlin2019bert} pretraining which uses Wikipedia and BookCorpus~\citep{zhu2015aligning} ($16$GB of texts) as the pretraining corpora. 
The encoder architecture is a $12$-layer Transformer, and the model dimension is $768$.
We train both absolute and relative position embeddings~\citep{raffel2019t5} in the encoder.
The decoder is a $4$-layer Transformer with the same model dimensions as the encoder.
Since the decoder is not used in downstream tasks, \method's encoder can be fairly compared with previous $12$-layer base-sized models.
The model is trained for $125$K steps with $2,048$ sequences per batch, which amounts to $256$M samples in total. 
The maximum input sequence length is $512$ tokens.
The vocabulary is constructed with BPE~\citep{sennrich2015neural} and consists of $32,768$ \emph{uncased} subword units.

The \textit{base++} setting follows RoBERTa~\citep{liu2019roberta} pretraining which extends the \textit{base} setting by incorporating larger pretraining corpora and training the same model architecture for longer.
Specifically, the following corpora are used along with Wikipedia and BookCorpus: OpenWebText~\citep{Gokaslan2019OpenWeb}, CC-News~\citep{liu2019roberta}, and STORIES~\citep{trinh2018simple}.
This expands the pretraining corpora to contain $160$GB texts.
The model is trained for $2$M steps with $2,048$ sequences per batch, which amounts to $4$B samples in total.
The \textit{base++} setting also expands the vocabulary size to $64,000$~\citep{unilmv2} by using \emph{cased} subword units.

The \textit{large++} setting extends the \textit{base++} setting by scaling up the encoder architecture to $24$ layers and $1,024$ model dimensions.
The decoder is still a $4$-layer Transformer with the same model dimensions as the encoder.
Due to the high cost of training large models, we train for $1$M steps (half of the \textit{base++} setting) with $2,048$ sequences per batch, which amounts to $2$B samples in total.
Note that this is also half of the pretraining data used in RoBERTa~\citep{liu2019roberta} and BART~\citep{Lewis2020BARTDS}.

\textbf{Computation Environment.}
The experiments in this paper are conducted on $64$ A100 GPUs.

\textbf{Masking.} 
For all pretraining settings, we apply $15\%$ random masks to input sequeces.
We do not use the trick in conventional MLM~\citep{devlin2019bert,liu2019roberta} that replaces $10\%$ of \mask tokens with the original ones and another $10\%$ with random tokens.
We also experiment with higher masking rates (\eg, $40\%$) which are shown to be beneficial in \cite{Wettig2022ShouldYM} for training large models, but they do not yield better results than the default $15\%$ masking rate in our experiments.
This is probably because \cite{Wettig2022ShouldYM} use an efficient pretraining recipe that is different from the standard pretraining setup, with a larger learning rate, a larger batch size, a shorter sequence length, and fewer training steps.

\textbf{Position Embedding.} 
We learn both absolute and relative position embeddings~\citep{raffel2019t5} in the encoder, and only learn absolute position embeddings in the decoder.

\textbf{Dropout.} 
During the pretraining of \method, dropout is applied to the encoder but not the decoder, which we find to slightly improve stability.

\section{Hyperparameter Settings}
\label{app:hyper}

\begin{table*}[t]
\small
\centering
\caption{Hyperparameters used in pretraining.}
\begin{tabular}{l*{3}{c}}
\toprule
Hyperparameter & \textit{base} & \textit{base++} & \textit{large++} \\
\midrule
Max Steps & 125K & 2M & 1M \\
Peak Learning Rate & 5e-4 & 2e-4 & 1e-4 \\
Batch Size & 2048 & 2048 & 2048 \\
Warm-Up Steps & 10K & 10K & 10K \\
Sequence Length & 512 & 512 & 512 \\
Relative Position Encoding Buckets & 32 & 64 & 128 \\
Relative Position Encoding Max Distance & 128 & 128 & 256 \\
Adam $\epsilon$ & 1e-6 & 1e-6 & 1e-6 \\
Adam ($\beta_1$, $\beta_2$) & (0.9, 0.98) & (0.9, 0.98) & (0.9, 0.98) \\
Clip Norm & 2.0 & 2.0 & 1.0 \\
Dropout & 0.1 & 0.1 & 0.1 \\
Weight Decay & 0.01 & 0.01 & 0.01 \\
\bottomrule
\end{tabular}
\label{tab:hp_pretrain}
\end{table*}

\begin{table*}[t]
\small
\centering
\caption{Hyperparameter ranges searched for fine-tuning on GLUE. GLUE small tasks include {CoLA}, {RTE}, {MRPC} and {STS-B}. GLUE large tasks include {MNLI}, {QQP}, {QNLI} and {SST-2}.}
\begin{tabular}{l*{2}{c}}
\toprule
Hyperparameter & GLUE Small Tasks Search Space & GLUE Large Tasks Search Space  \\
\midrule
Max Epochs & \{2, 3, 5, 10\} & \{2, 3, 5\} \\
\multirow{2}{*}{Peak Learning Rate} & \textit{base}/\textit{base++}: \{2e-5, 3e-5, 4e-5, 5e-5\} & \textit{base}/\textit{base++}: \{1e-5, 2e-5, 3e-5, 4e-5\} \\
& \textit{large++}: \{7e-6, 1e-5, 2e-5, 3e-5\} & \textit{large++}: \{5e-6, 7e-6, 1e-5, 2e-5\} \\
Batch Size & \{16, 32\} & 32 \\
Warm-Up Proportion & \{6\%, 10\%\} & 6\% \\
Sequence Length & 512 & 512 \\
Adam $\epsilon$ & 1e-6 & 1e-6\\
Adam ($\beta_1$, $\beta_2$) & (0.9, 0.98) & (0.9, 0.98)\\
Clip Norm & - & - \\
Dropout & 0.1 & 0.1 \\
Weight Decay & 0.01 & 0.01 \\
\bottomrule
\end{tabular}
\label{tab:hp_finetune_glue}
\end{table*}

\begin{table*}[t]
\small
\centering
\caption{Hyperparameter ranges searched for fine-tuning on SQuAD 2.0.}
\begin{tabular}{l*{1}{c}}
\toprule
Hyperparameter & SQuAD 2.0 Search Space \\
\midrule
Max Epochs & \{2, 3\}\\
\multirow{2}{*}{Peak Learning Rate} & \textit{base}/\textit{base++}: \{2e-5, 3e-5, 4e-5, 5e-5\} \\
& \textit{large++}: \{7e-6, 1e-5, 2e-5, 3e-5\}\\
Batch Size & \{16, 32\} \\
Warm-Up Proportion  & \{6\%, 10\%\} \\
Sequence Length & 512 \\
Adam $\epsilon$ & 1e-6 \\
Adam ($\beta_1$, $\beta_2$) & (0.9, 0.98) \\
Clip Norm & - \\
Dropout & 0.1 \\
Weight Decay & 0.01 \\
\bottomrule
\end{tabular}
\label{tab:hp_finetune_squad}
\end{table*}

We report the detailed hyperparameters used for pretraining in Table~\ref{tab:hp_pretrain}.
The hyperparameter search ranges of fine-tuning are shown in Tables~\ref{tab:hp_finetune_glue} and \ref{tab:hp_finetune_squad} for GLUE and SQuAD 2.0, respectively. 

For fair comparisons, the same set of hyperparameters (in both pretraining and fine-tuning) is used for \method, RoBERTa (Ours) and ablations. 
We follow previous pretraining studies~\citep{liu2019roberta} to report the medians of downstream task fine-tuning results under the same set of five different random seeds.

\section{More Evaluation Results}
\label{app:more_eval} 

\begin{wrapfigure}{r}{0.276\textwidth}
\centering
\includegraphics[width=\linewidth]{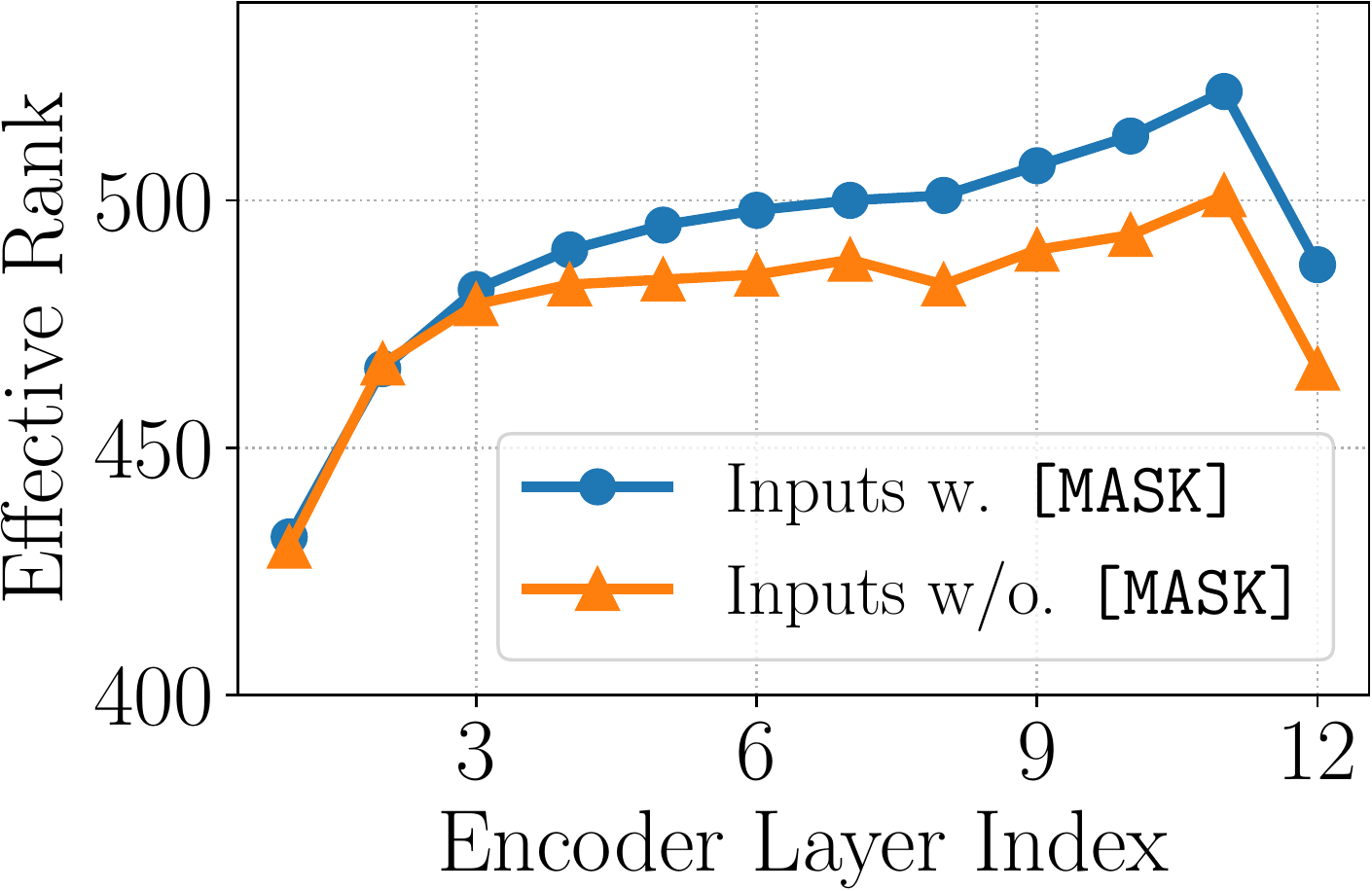}
\caption{With the original BERT masking strategy, the effective rank across layers for inputs without \mask and with \mask.\label{fig:eff_dim_bert_mask}}
\end{wrapfigure}
\textbf{BERT Masking Strategy.} In addition to our default masking strategy which directly applies $15\%$ random masks to input sequences, we also validate our findings under the original BERT masking strategy that replaces $10\%$ of \mask tokens with the original ones and another $10\%$ with random tokens. 
Figure~\ref{fig:eff_dim_bert_mask} demonstrates that the gap in effective representation rank between inputs with and without \mask under this setting is also notable, similar to the findings in Figure~\ref{fig:pretrain_finetune_dim}.
This confirms that randomly replacing a small percentage of \mask tokens with real tokens does not effectively address the representation deficiency issue, as the ratio of \mask tokens in pretraining is still high.

\textbf{Large Model Results.}
\begin{table*}[t]
\centering
\small 
\caption{
Standard single-task, single-model fine-tuning results (medians over five random seeds) evaluated on GLUE and SQuAD 2.0 development sets for large models.
$^\dagger$: \method is pretrained on half of RoBERTa/BART's data.
}
\resizebox{\textwidth}{!}{
\begin{tabular}{l*{9}{l}ll}
\toprule
\multirow{2}{*}{\textbf{Model}} & \multicolumn{9}{c}{\textbf{GLUE (Single-Task)}} & \multicolumn{2}{c}{\textbf{SQuAD 2.0}} \\ 
\cmidrule(lr){2-10}\cmidrule(lr){11-12}
& \textbf{MNLI-(m/mm)} & \textbf{QQP} & \textbf{QNLI} & \textbf{SST-2} & \textbf{CoLA} & \textbf{RTE} & \textbf{MRPC} & \textbf{STS-B} & \textbf{AVG} &
\textbf{EM} & \textbf{F1}\\
\midrule
\multicolumn{12}{c}{\textit{large++} setting: larger Transformer model trained on larger pretraining corpora ($160$GB)} \\ 
\midrule
BART 
& 89.9/90.1 & \textbf{92.5} & 94.9 & \textbf{96.6} & 62.8 & 87.0 & 90.4 & 91.2 & 88.2 & 86.1 & 89.2 \\ 
RoBERTa
& 90.2/90.2 & 92.2 & 94.7 & 96.4 & 68.0 & 86.6 & \textbf{90.9} & \textbf{92.4} & 88.9 & 86.5 & 89.4 \\
\method$^\dagger$
& \textbf{90.4/90.6} & 92.2 & \textbf{95.1} & 96.2 & \textbf{68.7} & \textbf{88.8} & 90.7 & 92.1 & \textbf{89.3} & \textbf{87.0} & \textbf{89.8} \\ 
\bottomrule
\end{tabular}
}
\vspace{-.5em}
\label{tab:large_model_res}
\end{table*}
We also show the performance of \method under larger model sizes in Table~\ref{tab:large_model_res}. 
Even trained on half of the pretraining data used in RoBERTa~\citep{liu2019roberta}, \method still performs comparably or better, demonstrating the potential of \method for larger models.

\textbf{Zero-Shot and Few-Shot Results.}
\begin{table*}[t]
\centering
\small 
\caption{
Zero-shot and few-shot performance. Few-shot results include mean and standard deviation (as subscripts) performance over 5 different training splits defined in \cite{gao2021making}. $^\dagger$: Results from \cite{gao2021making}.
}
\resizebox{\textwidth}{!}{
\begin{tabular}{l*{7}{l}ll}
\toprule
\multirow{2}{*}{\textbf{Model}} & \multicolumn{9}{c}{\textbf{GLUE (Single-Task)}} \\ 
& \textbf{MNLI-(m/mm)} & \textbf{QQP} & \textbf{QNLI} & \textbf{SST-2} & \textbf{CoLA} & \textbf{RTE} & \textbf{MRPC} & \textbf{STS-B} & \textbf{AVG} \\
\midrule
\multicolumn{10}{c}{\textit{zero-shot prompting}: direct inference on tasks via cloze-type MLM predictions} \\ 
\midrule
RoBERTa$^\dagger$
& $50.8/51.7$ & $49.7$ & $50.8$ & $83.6$ & $2.0$ & $51.3$ & $61.9$ & $-3.2$ & $43.4$ \\
\method
& $52.1/54.3$ & $52.0$ & $52.3$ & $83.5$ & $2.0$ & $54.5$ & $63.4$ & $-3.0$ & $44.7$ \\ 
\midrule
\multicolumn{10}{c}{\textit{head-based few-shot fine-tuning}: fine-tuning on $16$ samples per label with a linear classification head} \\ 
\midrule
RoBERTa$^\dagger$
& $45.8_{6.4}/47.8_{6.8}$ & $60.7_{4.3}$ & $60.2_{6.5}$ & $81.4_{3.8}$ & $33.9_{14.3}$ & $54.4_{3.9}$ & $76.6_{2.5}$ & $53.5_{8.5}$ & $58.4$ \\
\method
& $48.7_{4.5}/51.1_{6.0}$ & $64.5_{4.2}$ & $62.1_{6.1}$ & $81.2_{3.9}$ & $31.1_{13.9}$ & $58.0_{2.5}$ & $78.2_{2.1}$ & $53.0_{9.0}$ & $59.8$ \\ 
\midrule
\multicolumn{10}{c}{\textit{prompt-based few-shot fine-tuning}: fine-tuning on $16$ samples per label with cloze-type MLM templates} \\ 
\midrule
RoBERTa$^\dagger$
& $68.3_{2.3}/70.5_{1.9}$ & $65.5_{5.3}$ & $64.5_{4.2}$ & $92.7_{0.9}$ & $9.3_{7.3}$ & $69.1_{3.6}$ & $74.5_{5.3}$ & $71.0_{7.0}$ & $64.5$ \\
\method
& $70.7_{2.0}/73.3_{1.8}$ & $67.3_{4.6}$ & $65.1_{4.3}$ & $92.4_{1.1}$ & $14.3_{8.9}$ & $71.2_{3.3}$ & $74.8_{4.1}$ & $72.3_{6.5}$ & $66.2$ \\ 
\bottomrule
\end{tabular}
}
\vspace{-1em}
\label{tab:few_zero_res}
\end{table*}
Since \method is trained with the MLM objective, it is applicable to zero-shot and few-shot learning via prompt-based approaches. 
We report three groups of zero-shot/few-shot results on the GLUE tasks comparing MAE-LM (\textit{large++}) with RoBERTa (\textit{large++}) in Table~\ref{tab:few_zero_res}: (1) \textit{zero-shot prompting} which converts the classification tasks into cloze-type MLM predictions and directly uses pretrained models for inference on test sets; (2) \textit{head-based few-shot fine-tuning} which adds a linear classification head to the pretrained encoders for fine-tuning on $16$ samples per label; and (3) \textit{few-shot prompt-based fine-tuning} which fine-tunes the MLM models on tasks converted to cloze-type MLM formats with $16$ samples per label. 
We follow the basic manual prompt/label word setting and the training/development splits in \cite{gao2021making}.
For few-shot learning, the average and standard deviation over $5$ different training/development splits are reported.
Overall, MAE-LM can be combined with prompt-based methods for effective zero-shot and few-shot learning.

\section{More Discussions}
\label{app:discuss}
\textbf{Ethical Considerations.}
Despite their remarkable performance, pretrained models have been shown to come with risks such as exacerbating harmful biases~\citep{Bender2021OnTD,bommasani2021opportunities}. 
In our experiments, we follow the standard pretraining settings (\eg, data preparation, collection and preprocessing), and we expect more well-documented and filtered text corpora~\citep{Dodge2021DocumentingLW}, as well as future developments of harm reduction techniques~\citep{Liang2021TowardsUA} may help mitigate the ethical concerns about pretrained models.

\textbf{Connections to Prior Work.}
Since the advent of BERT~\citep{devlin2019bert}, there have been numerous developments in new pretraining and fine-tuning methods aiming to improve the effectiveness of pretrained models in downstream tasks.
The advantages of these proposed methods, however, are mostly demonstrated via empirical evidence alone, and our understanding of why certain methods are better than the others remains limited.
Our analyses in this work may advance the understanding of the benefits of some prominent methods:
ELECTRA~\citep{clark2020electra} fills \mask positions with real tokens; therefore, the encoder does not suffer from the representation deficiency issue. 
Different from the ablation in Section~\ref{sec:ablation} where we randomly sample real tokens to fill \mask, ELECTRA employs an MLM model to sample replaced tokens which are generally plausible alternatives to the original tokens, thus better preserving the contexts in pretraining.
These designs may help partially explain the effectiveness of ELECTRA.
Prompt-based methods~\citep{gao2021making,Schick2021ExploitingCF} adapt pretrained MLM models to downstream tasks by creating prompt templates that convert the target task into a masked token prediction problem.
This helps mitigate the representation deficiency issue that occurs in standard fine-tuning of MLM models as \mask tokens are also introduced into downstream data, resulting in more model dimensions being utilized.
Our findings may also shed light on certain previously observed phenomena in MLM models. For example, the rank deficiency issue might be responsible for the de-contextualization in self-attention patterns~\citep{Gong2019EfficientTO}.

\textbf{Implications on Autoregressive LMs.}
While autoregressive LM pretraining generally does not introduce artificial symbols such as \mask, our analyses can be easily extended to show that the representation deficiency issue can also arise in autoregressive pretraining when certain real tokens exist exclusively in the pretraining data but are either absent or occur infrequently in downstream data. Similar to the impact of \mask tokens, these tokens occupy dimensions during pretraining that may not be effectively utilized in downstream tasks. Consequently, it is desirable to maximize the vocabulary overlap between pretraining data and downstream data, which can be realized via pretraining data selection, training corpora pre-processing, and vocabulary pruning. We leave these explorations as future work.

\end{document}